\documentclass{article}
\usepackage{pgfplots}
\usepackage{styles/iclr2027_conference,times}
\usepackage{graphicx}
\usepackage{xcolor}            
\usepackage{amsmath}
\usepackage{amsfonts}
\usepackage{amssymb}
\usepackage{amsthm}
\usepackage{booktabs}
\usepackage{tcolorbox}
\usepackage{microtype}
\usepackage{algorithm}
\usepackage{algpseudocode}
\usepackage[hidelinks]{hyperref}
\usepackage[capitalise,noabbrev]{cleveref}
\usepackage{soul}
\pgfplotsset{compat=newest}
\usepackage{calc}
\usepackage{enumitem}
\usepgfplotslibrary{groupplots, fillbetween}

\newtheorem{theorem}{Theorem}[section]
\newtheorem{corollary}[theorem]{Corollary}
\newtheorem{proposition}[theorem]{Proposition}
\newtheorem{assumption}[theorem]{Assumption}
\theoremstyle{remark}
\newtheorem{remark}[theorem]{Remark}

\usepackage{tikz}
\usetikzlibrary{shapes,calc}
\usetikzlibrary{positioning, arrows.meta, backgrounds, fit, shapes.geometric}

\colorlet{netdraw}{blue!55!black}
\colorlet{netfill}{blue!8}
\colorlet{sigdraw}{black!65}
\colorlet{sigfill}{black!8}
\colorlet{hl}{orange!85!black}
\colorlet{acc}{teal!60!black}

\tikzset{
  net/.style={rectangle, rounded corners=2pt, draw=netdraw, fill=netfill,
              line width=0.9pt, minimum width=1.5cm, minimum height=1.3cm},
  sig/.style={rectangle, rounded corners=1pt, draw=sigdraw, fill=sigfill,
              line width=0.8pt, minimum width=0.34cm, minimum height=1.0cm},
  sigout/.style={sig, minimum height=0.7cm},
  lbl/.style={font=\scriptsize\sffamily, text=black!60},
  flow/.style={-{Stealth[length=2.4mm,width=1.8mm]}, line width=0.9pt, draw=black!70},
  hlnode/.style={draw=hl, fill=hl!12, line width=1.1pt},
  hlflow/.style={-{Stealth[length=2.2mm]}, line width=0.9pt, draw=hl, dashed,
                 dash pattern=on 3pt off 2pt},
  est/.style={rectangle, rounded corners=2pt, draw=acc, fill=acc!8,
              line width=0.8pt, font=\scriptsize, inner sep=4pt, align=center},
  samp/.style={circle, draw=acc, fill=acc!8, line width=1.1pt, minimum size=0.92cm},
  estflow/.style={-{Stealth[length=2mm,width=1.5mm]}, line width=0.75pt,
                  draw=acc, dash pattern=on 2.2pt off 1.6pt},
  brk/.style={draw=black!35, line width=0.7pt, dash pattern=on 1.6pt off 1.6pt},
}

\iclrfinalcopy 
\newcommand{\R}{\mathbb{R}}
\newcommand{\E}{\mathbb{E}}
\newcommand{\vecop}{\operatorname{vec}}
\newcommand{\Lcls}{\mathcal{L}_{\mathrm{cls}}}

\newcommand{\Lsmp}{L_{\mathrm{smp}}}
\newcommand{\Lcon}{L_{\mathrm{con}}}
\newcommand{\BW}{B_{W}}
\newcommand{\diag}{\operatorname{diag}}

\title{ProtoSeam: Lifting Classifier Training with Latent Gaussian Mixture Models}
\author{%
\parbox{\textwidth-2\tabcolsep}{\centering
Robert Lampel\,$^{1}$\thanks{Corresponding author: \texttt{robert.lampel@ovgu.de}}
\qquad Timon Klein\,$^{1}$
\qquad Sebastian Sager\,$^{1,2}$ \\[0.6em]
\mdseries
$^{1}$Department of Mathematics, Otto von Guericke University (Magdeburg, Germany) \\
$^{2}$Max Planck Institute for Dynamics of Complex Technical Systems (Magdeburg, Germany)
}}

\begin{document}
\maketitle

\begin{abstract}
We propose a lifted reformulation of supervised classification that
improves the final accuracy of standard classifiers without changing the
architecture at inference time. A network $N=N_2\circ N_1$ is split at
a single semantic interface and one learnable prototype per class is
inserted there. Training combines a quadratic consensus penalty that pulls
$N_1(x)$ toward the prototype of its class with a classification loss of $N_2$
evaluated on samples drawn around the prototypes, while no gradient crosses the
interface. At inference the prototypes are discarded and the unmodified
network $N_2\circ N_1$ is used. Across CIFAR-10, CIFAR-100, and TinyImageNet
with ResNet and vision transformer backbones, lifted training improves test accuracy by
up to six percentage points over variants without lifting under a shared tuning protocol.
Moreover, we provide theoretical justification for those results.
\end{abstract}

\section{Introduction}
\label{sec:intro}

Deep networks for supervised classification are usually trained end-to-end.
A single loss is backpropagated through the entire network $N$ in order to
maximize test accuracy. We instead view the network as a composition
$N = N_2 \circ N_1$, which introduces no additional parameters at inference time.
During end-to-end training, the two subnetworks are coupled in two ways.
First, by the chain rule, the gradient signal reaching the early subnetwork
$N_1$ must pass through the Jacobian of $N_2$. Whenever this Jacobian is
ill-conditioned or rank-deficient, this can lead to a distorted gradient signal
for $N_1$ is degraded.
Second, the classifier head $N_2$ is trained exclusively on the embeddings
$N_1(x)$ of the training samples. Nothing in the objective constrains its
behavior in a neighborhood of these embeddings, even though small
perturbations of the representation are precisely what the head encounters
as $N_1$ continues to evolve during training.

Prior work addressed the first coupling through
decoupled and local training, using per-layer auxiliary variables, local
error signals, or synthetic gradients (see \cref{sec:related}). However, these
methods target parallelism or reduced memory, and at best match the accuracy
of end-to-end training. Prototype- and margin-based losses, in contrast,
shape the embedding geometry but retain the end-to-end gradient path and
thus leave the second coupling in place.

Here we propose a lifted reformulation of classifier training
that removes both couplings at once, at a single, semantically meaningful
interface (the ``seam'').  The construction transfers Bock's direct multiple shooting method
\citet{bock1984multiple} from control
to classification. There, additional variables are introduced
together with matching conditions, leading to the same solution upon convergence.
Among other advantages, this formulation leads to better conditioning and convergence
gains~\citep{lampel2025liftings,lampel2026lifting}.
Here, we insert $n$ learnable lifted variables $s_1,\ldots,s_n$, one 
prototype per class, at the interface between $N_1$ and $N_2$ (\cref{fig:lifting}). Assuming
that the output of $N_1$ for class $i$ is normally distributed around $s_i$,
the training loss has two decoupled terms: 
\begin{enumerate}
	\item[(i)] a consensus penalty that drives $N_1(x)$ toward the prototype of its class, and
	\item[(ii)] a classification loss that trains $N_2$ on samples drawn from the class-conditional distribution around each prototype.
\end{enumerate}
An inter-class repulsion penalty keeps the prototypes
separated. No gradient crosses the interface, and by discarding the lifted
variables and reconnecting both networks we recover the unmodified composition $N_2\circ N_1$.

The reformulation delivers improved final test accuracy over both a matched 
end-to-end control and a standard baseline across six
dataset-architecture pairs (CIFAR-10, CIFAR-100, and TinyImageNet, each
with a ResNet and a custom small vision transformer (ViT-S) backbone), under a shared hyperparameter-tuning protocol
(\cref{tab:main}).  We also account for the gain theoretically
in \cref{sec:theory}).

\begin{figure}[t]
    \centering
    \includegraphics[width=0.97\linewidth]{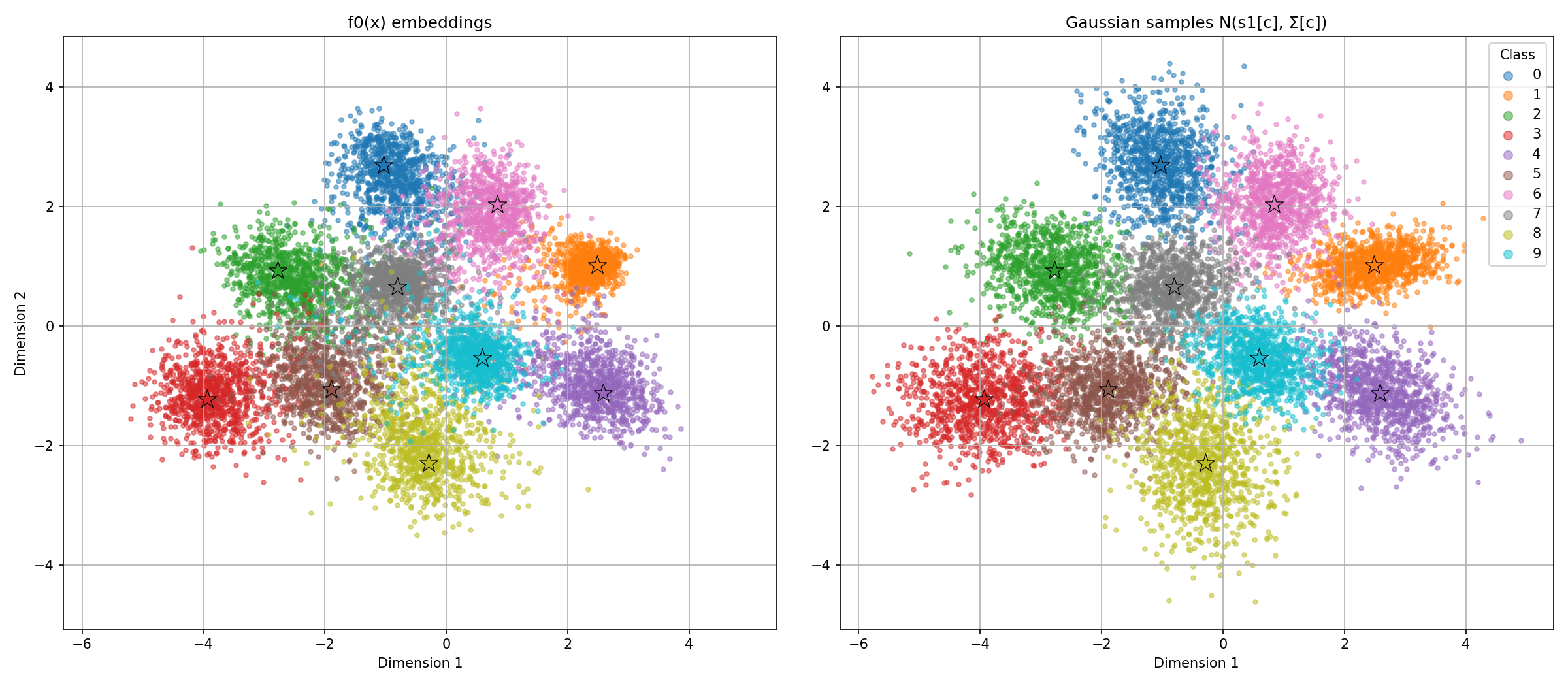}
    \caption{Output of the trained first network part $N_1$ (left) and the sampled input for the second part $N_2$ (right) for the MNIST \citep{deng2012mnist} dataset. The stars denote the position of the lifted variables, i.e., the class means of the sampled normal distribution. Here, $N_1$ consists of two convolutional layers, mapping to $\mathbb{R}^2$. $N_2$ consists of one fully connected layer from $\mathbb{R}^2$ to $\mathbb{R}^{10}$.}
    \label{fig:cluster}
\end{figure}

\paragraph{Contributions.}
\begin{enumerate}
  \item We formulate lifting at a single semantic interface of a
        classifier, with a Gaussian class-prototype structure at the lifting
        point and no gradient across it (\cref{sec:method},
        \cref{alg:lifted}).  The deployed architecture is unchanged.
  \item We demonstrate consistent accuracy improvements over both a matched
        end-to-end control and an unsplit baseline, up to six percentage points
        across six dataset-architecture pairs, together with a 
        $k,\rho_{\max}$ sweep on CIFAR-100 for the ViT-S backbone (\cref{sec:experiments}).
  \item We give a theoretical account of the gain (\cref{sec:theory}). An
        interface risk-transfer bound shows that the sampled classification
        loss is the deployed risk up to mismatch terms the objective itself
        controls (\cref{thm:risk-transfer}). A second-order analysis
        identifies sampling as an explicit sensitivity regularizer on the
        head (\cref{prop:smoothing}). Finally, we consider an exactly solvable one-dimensional
        model in which the selection effect is closed-form
        (\cref{prop:toy-noise-selection}).
\end{enumerate}

The Gaussian distributions we impose per class is empirically well-motivated by the
neural collapse phenomenon~\citep{papyan2020prevalence}. Deep classifiers
trained with cross-entropy converge to tight class-conditional clusters with
means arranged as a Simplex Equiangular Tight Frame during the terminal phase of
training.  Our lifted loss promotes a similar geometry from initialization rather than
waiting for it to emerge asymptotically (\cref{fig:cluster}). During this process
the within-class variance degenerates towards zero, conflicting with our approach of
sampling from the empirical covariance. We address this explicitly in \cref{sec:variance-floor}.

The term \emph{lifting} has been used for a multitude of ideas across various fields.
This work is conceptually distinct from the ``lifted neural networks'' of \citet{askari2018lifted},
which introduce auxiliary activation variables at
every layer to enable block-coordinate descent without imposing any
distributional structure.  It also differs from metric-learning methods that use
the word ``lifted'' in the sense of lifting pairwise distances to the full-batch
distance matrix~\citep{song2016deep}.  Our approach introduces a single set of
class prototypes at one semantically meaningful location, grounded both
in the multiple-shooting tradition and in the distributional assumptions of
variational latent-variable models.

\paragraph{Notation.}
Throughout, $n$ is the number of classes, $m$ the number of training samples,
$d$ the input dimension and $k$ the embedding (lifting) dimension.  We write
$[n]=\{1,\ldots,n\}$.  The network is split as $N=N_2\circ N_1$ with
parameters $\theta_1,\theta_2$; $N_1:\R^{d}\to\R^{k}$ and
$N_2:\R^{k}\to\R^{n}$.  Class prototypes are $s_1,\ldots,s_n\in\R^{k}$,
collected column-wise in $S\in\R^{k\times n}$.  For a symmetric positive
semi-definite $A$ we write $\kappa(A)=\lambda_{\max}(A)/\lambda_{\min}^{+}(A)$,
with $\lambda_{\min}^{+}$ the smallest positive eigenvalue, and
$\vecop$ denotes column-stacking vectorisation, so that
$\vecop(AB)=(B^\top\otimes I)\vecop(A)$.

\section{Related Work}
\label{sec:related}

\paragraph{Lifting and decoupled training.}
For the numerical solution of optimal control problems direct multiple shooting~\citep{bock1984multiple}
has become a standard approach. The decomposition into smaller subproblems together with matching
constraints has various advantages. Among those are superior stability, reduced nonlinearity, and
improved convergence speed. The latter aspect has been investigated for boundary value and optimal
control problems, including neural ODEs~\citep{lampel2025liftings,lampel2026lifting,chen2018neural}.
For multiple shooting the matching constraints hold exactly at the solution~\citep{nocedal2006numerical},
whereas we only enforce them using a penalty term in our final objective. The analogy is therefore
structural rather than one of constraint satisfaction.

The same idea appears in deep learning as per-layer auxiliary variables: lifted neural
networks~\citep{askari2018lifted} and their Fenchel
relaxation~\citep{gu2020fenchel}, the Method of Auxiliary
Coordinates~\citep{carreirawang2014mac}, and ADMM
training~\citep{taylor2016training}.  Related work removes the end-to-end path
using local error signals~\citep{nokland2019training}, greedy layer-wise
training~\citep{belilovsky2019greedy}, or predicted
gradients~\citep{jaderberg2017decoupled}. 
\citet{kang2019decoupling} first train the network end-to-end and then retrain
the classifier on frozen features.
Techniques from differential equations were transferred to
deep residual neural networks to enable layer-parallel training \citep{gunther2020layer}.
All these decouple every layer and none impose a structure on the latent representation.

\paragraph{Class prototypes and margin losses.}
Nearest-class-mean classifiers~\citep{mensink2013distance} and Prototypical
Networks~\citep{snell2017prototypical} evaluate an input based on its distance to a
class centroid, which is computed either post hoc or from an episode support set.
In constrast, our $s_i$ are optimized with the network and enter the losses
of both subnetworks. The closest related idea is center loss~\citep{wen2016discriminative},
which jointly learns a center $c_y$ per class and penalizes $\|f(x)-c_y\|_2^2$ alongside
softmax. The consensus term of \eqref{eq:combinedLoss} coincides with this
penalty. What is new in our approach is the removal of the gradient path from the class loss
back to the parameters of $N_1$ and the inter-class repulsion $\mathcal{P}$.
Finally, angular-margin objectives~\citep{liu2016large,liu2017sphereface,deng2019arcface}
constrain the classifier weights, while remaining end-to-end.

\paragraph{Gaussian embeddings and neural collapse.}
\citet{wan2018rethinking} proposed the Large-Margin Gaussian Mixture
(L-GM) loss, which models deep features as class-conditional Gaussians and adds a
likelihood term to the cross-entropy objective.  This is the closest existing
work in terms of the assumed distributional model assumed. Both L-GM and our method
impose per-class Gaussianity in the penultimate feature space.  The distinction
is that L-GM trains end-to-end, while our lifted loss decouples $N_1$ from $N_2$.
\citet{lee2018simple} confirm empirically that class-conditional
Gaussians fit penultimate-layer features well, using a Mahalanobis-distance
detector for out-of-distribution inputs.  The \emph{neural collapse} phenomenon
of \cite{papyan2020prevalence} further establishes that
deep classifiers naturally develop class-conditional clusters with means on a
Simplex ETF at the end of training. \cite{zhu2021geometric} proved the
global optimality of this structure under an unconstrained features model.
Closest in spirit to our approach of enforcing structure is \cite{yang2022inducing},
who \emph{fix} the classifier to a Simplex ETF and train only the backbone against it.
That work fixes the geometry a priori while we
learn the prototypes and additionally cut the gradient path between the two subnetworks.

\paragraph{Latent-variable models.}
We borrow the reparameterization trick from variational auto-encoders 
(VAE)~\citep{kingma2013auto}, and
the per-class Gaussian structure is shared with deep latent Gaussian
mixtures~\citep{nalisnick2016approximate,kingma2014semi} and with
FlowGMM~\citep{izmailov2020semi}. The latter enforces class-conditional Gaussianity in
a normalizing-flow latent space~\citep{rezende2015variational}. There are, however,
two important differences. First, flows require invertible architectures, incompatible with a standard
ResNet~\citep{he2016deep}, whereas our Gaussian structure is a soft penalty
applicable to any differentiable model. Second, unlike a VAE, our reparameterization is purely a sampling
device for training $N_2$.

\section{Method}
\label{sec:method}

Let $D = \{(x^{(j)}, i^{(j)})\}_{j=1}^{m}$ be a labeled dataset with samples
$x^{(j)} \in \R^{d}$ and labels $i^{(j)} \in [n]$, and let
$D_i = \{(x,l) \in D \mid l = i\}$ denote the samples of class $i$.  Our goal is
to train a network that assigns each sample its corresponding class.

\begin{figure}
\begin{center}
\begin{tikzpicture}[scale=1]
  \node[sig] at (0,0)   (in) {};
  \node[net] at (2.2,0) (na) {$N_1$};
  \node[net] at (4.4,0) (nb) {$N_2$};
  \node[sigout] at (6.6,0) (o)  {};
 
  \node[lbl] at (0,1.05)   {Input};
  \node[lbl] at (2.2,1.15) {First part};
  \node[lbl] at (4.4,1.15) {Second part};
  \node[lbl] at (6.6,1.05) {Output};
  \node[lbl]      at (0,-0.95)  {$x$ with label $1$};
 
  \draw[flow] (in) -- (na);
  \draw[flow] (na) -- (nb);
  \draw[flow] (nb) -- (o);
 
  \begin{scope}[on background layer]
    \node[fit=(na)(nb), draw=black!40, dashed, rounded corners=3pt,
          inner sep=7pt] (box) {};
  \end{scope}
  \node[lbl] at (3.3,-1.2) {$N = N_2 \circ N_1$};
\end{tikzpicture} \\[3.ex]
 
\begin{tikzpicture}
  \node[sig, hlnode] at (0,0)    (in3)  {};
  \node[net]         at (1.75,0) (na3)  {$N_1$};
  \node[sig, draw=acc, fill=acc!10, line width=1pt] at (3.3,0) (out1c) {};
  \node[sig]         at (5.16,0.16) (sn3) {};
  \node[sig]         at (5.08,0.08) (sm3) {};
  \node[sig, hlnode] at (5.00,0.00) (s1c) {};
  \node[samp]   at (6.85,0)  (z1)    {};
  \node[net]    at (8.6,0)   (nb3)   {$N_2$};
  \node[sigout] at (10.3,0)  (out2c) {};
 
  \begin{scope}[shift={(6.85,0)}]
    \draw[acc!55, line width=0.5pt] (-0.34,-0.2) -- (0.34,-0.2);
    \draw[acc, line width=0.85pt] plot[domain=-0.32:0.32,samples=33,smooth]
          (\x,{0.62*exp(-19*\x*\x)-0.2});
  \end{scope}
 
  \node[lbl] at (0,1.05)     {Input};
  \node[lbl] at (1.75,1.05)  {First part};
  \node[lbl] at (5.08,1.30)  {Lifted parameters};
  \node[lbl] at (5.08,1.0)  {(\textbf{Proto}types)};
  \node[lbl] at (8.6,1.05)   {Second part};
  \node[lbl] at (10.3,1.05)  {Output};
 
  \node[lbl] at (0,-0.95)    {$x$ with label $\textcolor{hl}{1}$};
  \node[lbl, align=center] at (3.3,-1.02)
    {first output\\[1pt] \textcolor{acc}{$N_1(x)$}};
  \node[lbl] at (5.00,-0.95) {$\textcolor{hl}{s_1},s_2,\dots,s_n$};
  \node[lbl, align=center] at (7.15,-1.10)
    {$\textcolor{acc}{z_1}=\textcolor{hl}{s_1}+\tilde\Sigma_1^{1/2}\xi$\\[1pt]
     $\xi\sim\mathcal{N}(0,I_k)$};
 
  \draw[flow] (in3) -- (na3);
  \draw[flow] (na3) -- (out1c);
  \draw[flow, draw=hl] (5.42,0) -- (z1);
  \draw[flow] (z1) -- (nb3);
  \draw[flow] (nb3) -- (out2c);
 
  \draw[brk] (4.05,-0.70) -- (4.05,0.70);
  \node[lbl, text=black!45, rotate=90] at (4.34,0) {no gradient};
  \node[lbl, text=black!45, rotate=90] at (3.76,0) {\textbf{Seam}};
 
  \node[est, text=acc] at (5.05,2.30) (cov)
    {\textbf{empirical covariance}: once per epoch, detached\\[1.5pt]
     $\hat\Sigma_1=\widehat{\operatorname{Cov}}\{N_1(x):x\in D_1\}$, regularize to $\tilde\Sigma_1$};
  \draw[estflow] (out1c.north) -- (cov.south -| out1c);
  \draw[estflow] (cov.south -| z1) -- (z1.north);
 
  \draw[hlflow] (0,-1.55) to[out=-90, in=-90, looseness=0.42]
    node[lbl, text=hl, below=2pt, midway]
      {sample of class $1$ $\Rightarrow$ consensus target $s_1$}
    (5.00,-1.55);
\end{tikzpicture}
\end{center}
\caption{\textbf{The lifting procedure.}  The original network (top) is split into two
  parts and one lifted variable is introduced per class (bottom).  The part
  $N_1$ is trained against $s_i$ by the consensus
  term (orange), while $N_2$ is trained on samples $z_i$ drawn around $s_i$, so
  no gradient crosses the interface (seam).  The class-conditional covariance (teal) is
  estimated once per epoch from the embeddings $N_1(x)$ of the whole class,
  regularized by \eqref{eq:shrinkage}, and then used as the
  scale of the sampling step \eqref{eq:reparam}.  At
  inference lifted variables and sampling step are discarded and the
  original path $N_2\circ N_1$ is restored.}
\label{fig:lifting}
\end{figure}

Let $N$ be a neural network which we split into two parts $N_1$ and $N_2$ such
that $N(x) = N_2(N_1(x))$.  Following the reparameterization trick that underlies
variational autoencoders~\citep{kingma2013auto}, we make the central assumption
that the output of $N_1$ is normally distributed within each class in the
embedding space $\R^{k}$. Writing $(X,Y)$ for a random labeled sample,
\begin{equation}
    N_1(X) \mid Y = i \ \sim \ \mathcal{N}(s_i, \Sigma_i),
    \qquad \forall\, i \in [n],
    \label{eq:classAssumption}
\end{equation}
where $s_i \in \R^{k}$ and $\Sigma_i \in \R^{k \times k}$ denote the
class-specific mean and covariance. Equation~\eqref{eq:classAssumption} is a
modeling assumption and does not hold for an arbitrary trained network. The
consensus penalty below is what makes it approximately true.

The training objective is twofold.  On the one hand, the prototypes $s_1, \ldots, s_n$ must
be positioned such that $N_2$ separates the confidence regions of the
distributions $\mathcal{N}(s_i, \Sigma_i)$.  On the other hand, the output of $N_1$ has to
actually follow \eqref{eq:classAssumption}.

\paragraph{Covariance estimation.}
\label{sec:variance-floor}

Rather than learning means and covariances jointly, we treat only the means
$s_1,\ldots,s_n$ as learnable parameters and re-estimate the covariances
empirically once per epoch from the class-conditional embeddings,
\begin{equation}
    \hat{\Sigma}_i = \frac{1}{|D_i|} \sum_{(x,\,\cdot\,) \in D_i}
    \bigl(N_1(x) - \hat{\mu}_i\bigr)\bigl(N_1(x) - \hat{\mu}_i\bigr)^\top,
    \qquad
    \hat{\mu}_i = \frac{1}{|D_i|} \sum_{(x,\,\cdot\,) \in D_i} N_1(x),
    \label{eq:covariance}
\end{equation}
so that each estimate is formed exclusively from the embeddings of the
corresponding class.  Note that \eqref{eq:covariance} centers on the empirical
mean $\hat\mu_i$ rather than on the prototype $s_i$.
Centering on $s_i$ instead would fold the consensus residual into the
covariance and inflate it, so we retain $\hat\mu_i$.

In practice, $\hat\Sigma_i$ is singular whenever $k \ge |D_i|$ and potentially ill-conditioned well before
that (\cref{sec:dimension}). Moreover, as the
consensus penalty tightens, $\hat\Sigma_i \to 0$, so the samples fed to $N_2$
degenerate to the $n$ points $s_1,\ldots,s_n$.  In that limit $N_2$ is trained on
a dataset of $n$ distinct inputs and its behavior between prototypes is
unconstrained, while at inference it receives genuine embeddings $N_1(x)\neq
s_i$. To account for this effect and potential numerical inaccuracies,
we always ensure that the estimate remains positive
definite by adding a scaled identity matrix, i.e.,
\begin{equation}
    \tilde\Sigma_i = 
    \,\hat\Sigma_i
      + \sigma_0^2 I_k,
    \qquad 
    \sigma_0>0,
    \label{eq:shrinkage}
\end{equation}
which then admits a stable Cholesky factorization.
The $\sigma_0$ can equivalently be interpreted as how much of the embedding space
around each prototype $N_2$ is required to classify correctly.

\paragraph{Inter-class repulsion.}

As an additional form of regularization, we require the class clusters to be
mutually well separated.  We therefore introduce the penalty
\begin{equation}
    \mathcal{P}(s_1, \ldots, s_n)
    = \sum_{\substack{i,j \in [n] \\ i < j}} \exp\bigl(-\alpha \|s_i - s_j\|_2\bigr),
    \qquad \alpha > 0,
    \label{eq:penalty}
\end{equation}
which decays rapidly once the means are far apart and thus penalizes only those
pairs that remain close. For all benchmarks we chose the fixed value of $\alpha=2$.
The main purpose of this penalty is to separate the class
clusters at the beginning of the training process. Without this penalty, all clusters
usually remain very close to each other and are never really separated.

\paragraph{The lifted objective.}

Let $\Lcls$ denote the chosen classification loss.  We define the lifted training
objective as
\begin{align}
    \mathcal{L}(\theta_1,\theta_2,S)
    = \underbrace{\frac{\rho}{2m}\sum_{(x,i)\in D}\bigl\|N_1(x) - s_i\bigr\|_2^2}
        _{\text{consensus, trains }\theta_1\text{ and }S}
    &+ \underbrace{\frac{1}{n}\sum_{i=1}^{n}
        \E_{\xi}\bigl[\Lcls\bigl(N_2(z_i),\, i\bigr)\bigr]}
        _{\text{classification, trains }\theta_2\text{ and }S} \notag \\
    &+ \rho\,\mathcal{P}(s_1,\ldots,s_n),
    \label{eq:combinedLoss}
\end{align}
where the embedding fed to $N_2$ is sampled via the reparameterization
\begin{equation}
    z_i = s_i + \tilde\Sigma_i^{1/2} \xi, \qquad \xi \sim \mathcal{N}(0, I_k),
    \label{eq:reparam}
\end{equation}
with $\tilde\Sigma_i^{1/2}$ the Cholesky factor of \eqref{eq:shrinkage}, treated
as a constant with respect to $\theta_1$ within an epoch.  The first term pulls
the embedding of each sample toward the prototype of its class; the second trains
$N_2$ on samples drawn from the assumed class distribution itself.

Three points deserve emphasis.

\textbf{(i) The repulsion weight is shared.}  In our implementation the repulsion
strength is given by $\rho \mathcal{P}(\cdots)$, tied to the penalty $\rho$ of the least-squares term. 
Theoretically, these two terms are independent and introducing a separate scale for $\mathcal{P}$
might be worthwhile. However, since we perform a large ablation study over $\rho$,
we believe our formulation to be more practical.

\textbf{(ii) The classification term does not depend on $x$.}  Because $z_i$ is a
function of $(s_i,\tilde\Sigma_i,\xi)$ only, the second term of
\eqref{eq:combinedLoss} depends on the data solely through the class index.
This is the precise sense in which the two subnetworks are decoupled: the
$\theta_2$-subproblem is an $n$-component Gaussian classification problem whose
cost is independent of $m$, and it may be optimized with as many Monte Carlo
draws per class as desired, independently of the minibatch composition.  We use
the uniform class weighting $\tfrac1n$ in \eqref{eq:combinedLoss}; weighting by
empirical frequency $|D_i|/m$ is the alternative appropriate for imbalanced data.

\textbf{(iii) $\rho$ has two readings.} On the one side, it acts as a conventional penalty weight.
On the other side, $\tfrac{\rho}{2}\|N_1(x) - s_i\|_2^2$ is the negative log-density of an
isotropic Gaussian with variance $\rho^{-1}$ (up to an additive constant). Hence, its
reciprocal quantifies the tolerated spread around each prototype. 

This lifting procedure is applied during training only.  For validation
and at inference time the lifted variables are discarded and the two subnetworks
are reconnected, recovering the original architecture $N = N_2 \circ N_1$.

\begin{algorithm}[t]
\caption{Lifted training of $N = N_2\circ N_1$}
\label{alg:lifted}
\begin{algorithmic}[1]
\Require dataset $D$, epochs $T$, schedule $\rho(\cdot)$, 
         floor $\sigma_0$ 
\State initialize $\theta_1,\theta_2$; initialize $S$ (\cref{sec:init}); set
       $\tilde\Sigma_i \gets \sigma_0^2 I_k$
\For{$t = 1,\dots,T$}
  \For{each minibatch $B\subseteq D$}
    \State $\mathcal{L}_{\mathrm{con}} \gets \frac{\rho(t)}{2|B|}
            \sum_{(x,i)\in B}\|N_1(x)-s_i\|_2^2$
    \State draw $\xi_i\sim\mathcal{N}(0,I_k)$ and set
            $z_i \gets s_i + \tilde\Sigma_i^{1/2}\xi_i$ for $i\in[n]$
           \Comment{$\tilde\Sigma_i$ detached}
    \State $\mathcal{L}_{\mathrm{cls}} \gets
            \frac1n\sum_{i=1}^{n}\Lcls(N_2(z_i), i)$
    \State update $\theta_1,\theta_2,S$ on
            $\mathcal{L}_{\mathrm{con}}+\mathcal{L}_{\mathrm{cls}}
             +\rho\,\mathcal{P}(S)$
  \EndFor
  \State estimate $\hat\Sigma_i$ over $D$ and regularize
         $\tilde\Sigma_i$ by \eqref{eq:shrinkage} if necessary
\EndFor
\State \Return $N_2\circ N_1$ \Comment{$S$, $\tilde\Sigma$ discarded}
\end{algorithmic}
\end{algorithm}

The per-epoch empirical covariance computation in \cref{alg:lifted} costs one additional
forward pass over the training set, i.e.\ roughly a $1/3$ overhead per epoch
relative to a standard forward--backward pass, plus $O(nk^2)$ memory for the
factors.

There are alternative and considerably cheaper approaches to compute the covariance.
For example, we could compute and add up the dyadic products of $N_1(x)-s_i$ during the training process
to obtain running estimates of the class covariances for the next epoch. In our experiments, this produced almost
identical results but occasionally led to numerical difficulties whenever the output distribution
of $N_1$ changed rapidly during one epoch. 

An even cheaper approach, with practically no additional cost,
would be to estimate the covariance using a scaled identity matrix. For the right scaling of the identity matrices
this once again produced similar results to the empirical covariance computation. 
Based on the interpretation of $\rho$, choosing a scaling of $\rho^{-1}$ would be the obvious choice.
We discuss the reason why this simple approximation works so well in \cref{sec:results}.

For our ProtoSeam benchmarks, however, we use an additional pass to compute the empirical covariance after every epoch,
unless stated otherwise. 
We made this choice for better interpretability and to avoid introducing another hyperparameter.

\subsection{Design Choices}
\label{sec:design}

The formulation above leaves several choices open which we discuss in the following and investigate with numerical sweeps
reported in \cref{sec:experiments}.

\paragraph{Lifting location.}
The lifting point should not immediately follow a ReLU, since the output of $N_1$ would
then be confined to the non-negative orthant. In this case, a cluster whose prototype lies
on the boundary of that region cannot be symmetrically surrounded by the Gaussian
of \eqref{eq:classAssumption}.  Hence, the model is misspecified exactly where the
consensus penalty is tightest.  We therefore always lift before such an
activation or before the final classification layer. Given our assumption that
every class is normally distributed around one mean, the network $N_2$ need not
be very complex to separate these clusters. In fact, we only use a single fully
connected layer for $N_2$. Using more complex modeling assumptions, such as a 
Gaussian mixture for every class, might necessitate larger $N_2$.
We address the detailed construction in~\cref{sec:experiments}.

\paragraph{Lifting dimension.}
\label{sec:dimension}
Our approach assumes that the output of $N_1$ is normally distributed within each
class and estimates the corresponding covariance from data.  A non-singular
sample covariance in dimension $k$ requires at least $k+1$ samples per class, so
$\min_i |D_i| > k$ is a hard upper bound on the embedding dimension. Regularization
might rectify this issue, but accuracy requires considerably more.
For sub-Gaussian class-conditional embeddings,
the sample covariance $\hat\Sigma_{m'}$ formed from $m'$ samples
satisfies~\cite[Thm.~4.7.1 and Rem.~4.7.2]{vershynin2018high}
\begin{equation}
    \E\bigl\|\hat{\Sigma}_{m'} - \Sigma\bigr\|
    \;\le\; C\Bigl(\sqrt{\tfrac{k}{m'}} + \tfrac{k}{m'}\Bigr)\,\|\Sigma\|,
    \label{eq:vershynin}
\end{equation}
with an absolute constant $C$ depending only on the sub-Gaussian norm, so a
relative error $\varepsilon$ requires $m' \gtrsim \varepsilon^{-2} k$.
This bound is very restrictive, applying it to CIFAR-100 with $m' = 500$ training images per class results for 
$\varepsilon = 0.1$ in $k \lesssim 5$. In our experiments we observe better performance for way higher dimensions than this. We conjecture that this is a combined effect of the additional repulsion term and the covariances tending towards zero, rendering the relative error bound \eqref{eq:vershynin} unimportant.

\paragraph{Initialization of class means.}
\label{sec:init}
The lifted variables admit custom initialization.  The natural choice is the
empirical class means $s_i = \E_{X \mid Y=i}[N_1(X)]$ computed from the
initial network. Following \citet{lampel2025liftings,lampel2026lifting},
we refer to this approach as Forward Sweep Initialization (FSInit).

Using FSInit is motivated by the observation that semantically similar classes are also close to each other in the embedding space,
compare Figure~\ref{fig:cluster}. Therefore, we expect the FSInit outcome to be an approximation of this final positioning, thus avoiding
having to reorder the class clusters. However, the benefits clearly
depend on the initialization of the network weights $\theta_1$. We illustrate the position of the
prototypes over the course of training in \cref{fig:class_centers_mnist}.

Alternative initializations are based on random or user-guided values.
Also unit vectors can be used if the embedding dimension $k$ is equal to the number $n$ of classes. 
We only observed a minor impact of initialization on the final accuracy, with FSInit performing best.

\paragraph{Penalty annealing.}
\label{sec:annealing}
Rather than holding $\rho$ fixed, we increase it depending on epoch $t$ as
\begin{equation}
\rho(t) = \rho_{\min} + (\rho_{\max}-\rho_{\min})
  \sin\!\left(\frac{\pi}{2}\cdot\frac{t-1}{T-1}\right),
  \qquad t = 1,\dots,T .
\label{eq:anneal}
\end{equation}
This increases the penalty fastest towards the beginning of training, then slowly approaches $\rho_{\max}$ for larger $t$. 
Using smaller values of $\rho(t)$ in the early training phase is motivated by allowing the clusters to separate rather than collapse immediately.
For later training epochs $t$ the enforcement of consensus becomes a higher priority.

\section{Experiments}
\label{sec:experiments}

\subsection{Protocol}
For a given dataset $\mathcal{D}$ and architecture family $M$, we train and
compare three variants under a common recipe, so that any accuracy gap can be
attributed to the lifting mechanism rather than to confounding architectural or
optimization differences.

\begin{description}[
  style=sameline,
  leftmargin=\widthof{\textbf{Unlifted.}}+1em,
  labelwidth=\widthof{\textbf{Unlifted.}},
  labelsep=1em
]
  \item[Baseline.] An independent, non-split reference network is trained end-to-end in the standard way, providing an external
    calibration point against a conventional model of the same family.
  \item[ProtoSeam.] The architecture is partitioned into a feature extractor $N_1$
    and a classifier head $N_2$, coupled through the per-class prototype matrix
    $S\in\R^{k\times n}$.  Parameters $\theta_1,\theta_2,S$ are optimized under
    \eqref{eq:combinedLoss} with penalty weight $\rho$. We lift just before the
    classification layer, meaning that the last fully connected layer of the 
    original architecture (the baseline) is replaced by one mapping to $\mathbb{R}^k$ to form
    $N_1$. Consequently, $N_2$ consists of only one fully connected layer mapping
    from $\mathbb{R}^{k}$ to $\mathbb{R}^n$.
  \item[Unlifted.] 
    The ProtoSeam architecture is trained end-to-end by
    ordinary backpropagation, without consensus penalty, the repulsion term and
    the auxiliary variable $S$. This separates the effect of the lifting mechanism from those
    of the minor architecture change between Baseline and ProtoSeam (remember that the final fully connected layer
    in Baseline mapping to $\mathbb{R}^n$ is replaced by two fully connected layers; the first maps to $\mathbb{R}^k$
    and the second from $\mathbb{R}^k$ to $\mathbb{R}^n$, with no activation function in between).
\end{description}

\paragraph{Sweep.} We perform an exemplary sweep across both the embedding dimension $k$ and the penalty $\rho_{\max}$ for the ViT-S on CIFAR-100 in \cref{fig:penalty-sweep}.
Tuning both these values for every model type and dataset would increase the lifted accuracy further, but also entail a further expensive hyperparameter search.
Therefore, we chose an embedding dimension of $k=32$ and a penalty of $\rho_{\max}=16$ for all further benchmarks in~\cref{sec:results}.

\paragraph{Common training recipe.} All variants use the same splits, epoch
budget $T$, and optimizer family (SGD with Nesterov momentum, weight decay, short
linear warm-up, cosine-annealed decay).

\begin{remark}[Hyperparameter tuning]
\label{rem:confound} To ensure that the better accuracy is not merely an artifact
of different hyperparameters, we ran a large grid search to determine the best learning rate
and weight decay for both the lifted and the baseline variant, for every dataset and model.
For the unlifted variant we used the same hyperparameters as for the baseline, since both train end-to-end.
The determined optimal parameters are summarized in~\cref{tab:hyperparameters}.
\end{remark}

\subsection{Datasets and architectures}
We consider a ResNet~\citep{he2016deep} and a custom small Vision Transformer (ViT-S)~\citep{dosovitskiy2021image}
backbone on CIFAR-10, CIFAR-100~\citep{krizhevsky2009learning}, and
TinyImageNet~\citep{le2015tiny}. On the one hand these models and datasets are complex
enough so that there are still meaningful gains in accuracy to be made. On the other hand,
they are small enough to make an exhaustive comparison across a wide range of hyperparameters.

\definecolor{sweepblue}{RGB}{33,94,163}
\definecolor{sweeporange}{RGB}{217,95,2}

\begin{figure}[h]
\centering
  \begin{tikzpicture}
  \begin{groupplot}[
    group style={group size=2 by 1, horizontal sep=0.55cm, yticklabels at=edge left},
    width=0.54\linewidth,
    height=5.cm,
    scale only axis=false,
    ymin=58, ymax=74,
    ytick={60,64,68,72},
    xmode=log, log basis x=2,
    log ticks with fixed point,
    ylabel={Test accuracy (\%)},
    ylabel style={font=\small},
    xlabel style={font=\small},
    title style={font=\small, yshift=-3pt},
    tick label style={font=\footnotesize},
    tick align=outside, tick pos=left,
    major tick length=2.5pt,
    grid=major,
    major grid style={line width=0.3pt, draw=black!12},
    axis line style={line width=0.5pt, draw=black!70},
    every axis plot/.append style={line width=0.9pt},
    error bars/error bar style={line width=0.6pt},
    error bars/error mark options={rotate=90, mark size=2pt, line width=0.6pt},
    legend style={font=\footnotesize, draw=none, fill=white, fill opacity=0.9,
                  text opacity=1, row sep=-1pt, inner sep=2pt},
    legend cell align=left,
  ]
  \nextgroupplot[
      title={(a) Fixed $k = 32$},
      xlabel={Terminal penalty weight $\rho_{\max}$},
      xtick={1,2,4,8,16,32},
      xticklabels={1,2,4,8,16,32},
      xmin=0.8, xmax=40,
  ]
  \addplot[color=sweepblue, mark=*, mark size=2pt,
           error bars/.cd, y dir=both, y explicit]
    table[x=x, y=y, y error plus=ep, y error minus=em] {
      x        y        ep      em
      0.97     59.6733  0.3667  0.4733
      1.94     64.0867  0.0733  0.1267
      3.88     67.8600  0.6200  0.5400
      7.76     69.5733  0.7267  0.3933
      15.52    71.2467  0.4533  0.4867
      31.04    71.4933  0.8067  0.5533
    };

  \addplot[color=sweepblue, dashed, mark=square*, mark size=2pt,
           mark options={solid, fill=white},
           error bars/.cd, y dir=both, y explicit, error bar style={solid}]
    table[x=x, y=y, y error plus=ep, y error minus=em] {
      x        y        ep      em
      1.03     60.2467  0.3233  0.2667
      2.06     64.0667  0.1933  0.1867
      4.12     67.5167  0.7933  0.6967
      8.24     69.3267  0.5633  0.7267
      16.48    70.4333  0.7067  0.7533
      32.96    71.2533  0.7167  0.7733
    };
  \addplot[only marks, mark=o, mark size=3.6pt, color=sweeporange, line width=1.1pt]
    coordinates {(15.52,71.2467) (16.48,70.4333)};
  \nextgroupplot[
      title={(b) Fixed $\rho_{\max} = 16$},
      xlabel={Embedding dimension $k$},
      xtick={8,16,32,64,128},
      xticklabels={8,16,32,64,128},
      xmin=6.4, xmax=160,
      ylabel={},
      legend pos=south east,
  ]
  \addplot[color=sweepblue, mark=*, mark size=2pt,
           error bars/.cd, y dir=both, y explicit]
    table[x=x, y=y, y error plus=ep, y error minus=em] {
      x        y        ep      em
      7.76     66.2667  0.1933  0.1067
      15.52    70.0867  0.5333  0.7667
      31.04    71.2467  0.4533  0.4867
      62.08    70.9733  0.5267  0.4733
      124.16   71.2400  0.3400  0.2000
    };
  \addlegendentry{validation (held-out)}
  \addplot[color=sweepblue, dashed, mark=square*, mark size=2pt,
           mark options={solid, fill=white},
           error bars/.cd, y dir=both, y explicit, error bar style={solid}]
    table[x=x, y=y, y error plus=ep, y error minus=em] {
      x        y        ep      em
      8.24     66.4567  0.1133  0.1167
      16.48    69.5233  0.4467  0.2333
      32.96    70.4333  0.7067  0.7533
      65.92    70.9700  0.6900  0.4100
      131.84   71.1800  0.1500  0.2600
    };
  \addlegendentry{test}
  \addplot[only marks, mark=o, mark size=3.6pt, color=sweeporange, line width=1.1pt]
    coordinates {(31.04,71.2467) (32.96,70.4333)};
  \addlegendentry{setting used in Sec.~4.3}
\end{groupplot}
\end{tikzpicture}

\caption{ViT-S on CIFAR-100, using the empirically computed covariance. 
On the left-hand side we compare the test accuracy for different values of
$\rho_{\max}$ and a fixed embedding dimension of $k=32$.
The right plot shows the test accuracy for different embedding dimensions,
this time for a fixed penalty of $\rho_{\max}=16$.
For both ablations we observe a monotonic climb which first increases steeply
and then stabilizes. We use $\rho_{\max}=16$ and $k=32$ in the following, which we found to work best across all datasets.}
\label{fig:penalty-sweep}
\end{figure}

\subsection{Results}
\label{sec:results}
We observe a consistent trend across both architectures. Our method always beats both the baseline and the unlifted variant. Therefore, we can rule out the possibility that the advantage comes from the altered architecture.
For TinyImageNet, the accuracy improves by more than five percentage points compared to the baseline for ViT-S.

We show an exemplary training run in \cref{fig:tin-vit-training-curves}, which presents a typical pattern for all training runs. At first, our method trails behind, while the output of $N_1$ is still noisy and the class prototypes are being positioned. After a while, our lifted formulation then overtakes both other variants and maintains its lead.

All technical details are described in~\cref{sec:tech_details}.

  \begin{table}[h]
  \centering
  \setlength{\tabcolsep}{4pt}
  \renewcommand{\arraystretch}{1.15}
  \small
  \begin{tabular}{lllrrrrr}
  \toprule
  Dataset & Model & Variant & seed 42 & seed 43 & seed 44 & Mean & Spread \\
  \midrule
  CIFAR-10 & ResNet-8 & Baseline & $94.88$ & $95.09$ & $94.77$ & $94.91$ & $0.32$ \\
   &  & Unlifted & $95.09$ & $95.38$ & $95.00$ & $95.16$ & $0.38$ \\
   &  & ProtoSeam & $95.27$ & $95.54$ & $95.57$ & $\mathbf{95.46}$ & $0.30$ \\
  \cmidrule(lr){2-8}
   & ViT-S & Baseline & $89.45$ & $89.50$ & $89.34$ & $89.43$ & $0.16$ \\
   &  & Unlifted & $90.55$ & $89.94$ & $89.66$ & $90.05$ & $0.89$ \\
   &  & ProtoSeam & $91.35$ & $91.64$ & $91.26$ & $\mathbf{91.42}$ & $0.38$ \\
  \midrule 
  CIFAR-100 & ResNet-8 & Baseline & $76.76$ & $76.96$ & $76.55$ & $76.76$ & $0.41$ \\
   &  & Unlifted & $75.61$ & $75.49$ & $75.30$ & $75.47$ & $0.31$ \\
   &  & ProtoSeam  & $78.18$ & $78.21$ & $78.63$ & $\mathbf{78.34}$ & $0.45$ \\
  \cmidrule(lr){2-8}
   & ViT-S & Baseline & $66.11$ & $65.63$ & $65.99$ & $65.91$ & $0.48$ \\
   &  & Unlifted & $66.58$ & $66.63$ & $66.54$ & $66.58$ & $0.09$ \\
   &  & ProtoSeam & $70.48$ & $71.00$ & $70.12$ & $\mathbf{70.53}$ & $0.88$ \\
  \midrule
  TinyImageNet & ResNet-8 & Baseline & $63.17$ & $63.13$ & $63.25$ & $63.18$ & $0.12$ \\
   &  & Unlifted & $58.45$ & $59.71$ & $59.16$ & $59.11$ & $1.26$ \\
   &  & ProtoSeam & $65.64$ & $65.48$ & $65.69$ & $\mathbf{65.60}$ & $0.21$ \\
  \cmidrule(lr){2-8}
   & ViT-S & Baseline & $46.42$ & $47.45$ & $47.01$ & $46.96$ & $1.03$ \\
   &  & Unlifted & $48.54$ & $47.48$ & $47.76$ & $47.93$ & $1.06$ \\
   &  & ProtoSeam & $53.95$ & $53.65$ & $54.01$ & $\mathbf{53.87}$ & $1.07$ \\
  \bottomrule
  \end{tabular}
  \caption{Multi-seed confirmation (seeds 42/43/44): test accuracy (\%) for
  baseline, unlifted, and ProtoSeam across both 
  architectures and all three datasets. \textbf{Bold} = lifted's mean, which
  beats both comparators at every individual seed on every combination.}
  \label{tab:main}
  \end{table}

\paragraph{Ablation of the ProtoSeam components.}
For a better understanding of where the improvements in accuracy come from, we ran ablation studies to separate the contributions of penalty terms, sampling with a Gaussian distribution, and gradient cut, respectively.

To study the impact of penalty terms, we investigate one further variant which equals Unlifted, but includes all penalty terms and trainable class means $s_i$ from ProtoSeam. More precisely, the loss of the unlifted network is the same as \eqref{eq:combinedLoss} but with $z_i$ replaced by $N_1(x)$ in the classification term. Apart from the regularizing effect of sampling from a Gaussian distribution, the largest advantage of the lifted variant compared to this variant can be expected when $N_2$ is ill-conditioned. For our $N_2$ consisting of one layer we would therefore expect similar performance to the lifted version. We show all results in \cref{tab:unlifted_penalty}. 
On the four CIFAR pairs, ProtoSeam’s gains come mainly from the consensus and repulsion terms. On TinyImageNet, the gradient cut and the sampling strategy give further improvements of +0.50 and +2.15 points for ResNet-8 and ViT-S, respectively. Even without the gradient cut, the combination of a consensus penalty with inter-class repulsion is novel, as discussed in \cref{sec:related}.

Second, we investigate whether adding Gaussian noise to the unlifted variant can improve accuracy. The results for different levels of perturbation are summarized in \cref{tab:noise_ablation}. Without lifting, there seems to be little to no benefit for this type of regularization. Furthermore, it is conceptually different from ProtoSeam, which samples from a Gaussian distribution with approximately the same covariance.

\paragraph{Covariance estimation.} As mentioned in \cref{sec:method}, using an isotropic covariance estimate is a very efficient alternative. We compare the accuracy of this variant, using $I / 16$ as the isotropic estimate, with the empirical variance in \cref{tab:CovarianceAblation}. Both variants perform almost identically and the isotropic covariance is even slightly more accurate for ViT-S on CIFAR-10. The reason behind this good performance is that both models learn well-separated clusters with small covariances. To provide further evidence, we investigated how well the clusters are separated for ProtoSeam compared to the unlifted and baseline variant in \cref{tab:cluster_separation}. All examined metrics clearly show that ProtoSeam leads to a better separation.

Note that all theoretical results from \cref{sec:theory} still hold for the case of isotropic covariance. Theorem~\ref{thm:risk-transfer} holds for any sampling covariance, and the isotropic variant is the special case $\hat{\Sigma}_i=0$ of Corollary~\ref{cor:controlled}, with $\sigma_0^2$ equal to the chosen scale. The only change is how the covariance term is read: for the empirical variant it measures the estimation error; for the isotropic variant, it measures the deliberate misspecification.

\paragraph{Extension to larger models and datasets.}
To investigate the aspect of scalability, we trained ResNet-34 on ImageNet \citep{deng2009imagenet}. We took the learning rate and weight decay commonly used for the baseline in the literature (learning rate $0.1$, and weight decay $10^{-4}$) and used them for all variants.

Our modeling assumption is that $N_1$ can map each class to a single cluster. On ImageNet, ResNet-34 does not achieve this for every class. For example, the output of $N_1$ for class 592 (hard disk) shows two distinct clusters: one for disks inside a computer and one for disks outside of it. We visualize this in \cref{fig:imagenet_hard_disks} and \ref{fig:imagenet_hard_disks_samples} in the appendix.

On full ImageNet, the lifted ResNet-34 falls behind the baseline for all values of $k$ and $\rho_{\max}$
 that we tried: $\approx73.2\%$ for the baseline vs. $\approx71.3\%$ for the best lifted variant. The picture reverses, however, when we restrict ImageNet to the 294 classes whose latent distribution exhibits a single cluster. We call this subset GaussianImageNet and describe the exact selection procedure in \ref{sec:tech_details}. It is still the largest dataset in our evaluation and has a higher resolution than the others. We train ResNet-34 on GaussianImageNet using the aforementioned learning rate and weight decay and summarize the results in \cref{tab:GaussianImageNet}. We did not perform a full search over $\rho_{\max}$ and $k$; we only decreased $\rho_{\min}$
 to avoid an exploding penalty term at the start of training.

 \begin{table}
  \centering
  \setlength{\tabcolsep}{4pt}
  \renewcommand{\arraystretch}{1.15}
  \small
  \begin{tabular}{lllllrrrrr}
  \toprule
  Variant & lr & wd & $D$ & $\rho$ (min $\to$ max) & Seed 42 & Seed 43 & Seed 44 & Mean & Spread \\
  \midrule
  Baseline  & $0.1$ & $1\mathrm{e}{-4}$ & --   & --                 & $85.92$ & $85.87$ & $85.65$ & $85.81$ & $0.27$ \\
  Unlifted  & $0.1$ & $1\mathrm{e}{-4}$ & $64$ & --                 & $85.59$ & $85.22$ & $85.48$ & $85.43$ & $0.37$  \\
  ProtoSeam & $0.1$ & $1\mathrm{e}{-4}$ & $64$ & $0.001 \to 32$     & $86.62$ & $86.62$ & $86.67$ & $\mathbf{86.64}$ & $0.05$  \\
  \bottomrule
  \end{tabular}
  \caption{Test accuracy ($\%$) for the baseline, unlifted, and ProtoSeam variant of ResNet-34 on GaussianImageNet. ProtoSeam beats the other two variants consistently across all seeds.}
  \label{tab:GaussianImageNet}
\end{table}

\section{Conclusion}
\label{sec:conclusion}
We introduced a new technique for training neural networks for classification problems.
By splitting a network $N=N_2\circ N_1$ at a single interface and inserting one Gaussian prototype per class, 
we formulated an objective in which $N_1$ is trained toward the prototype of its class and $N_2$ on samples
drawn around it, with no gradient crossing the seam. 
At inference the prototypes are discarded, and the unmodified network is deployed at no additional cost.
Across CIFAR-10, CIFAR-100 and TinyImageNet with ResNet-8 and ViT-S backbones,
lifted training improved test accuracy over both a baseline model and the
matched end-to-end unlifted control on every dataset--architecture pair and at every
seed, by up to $6.91$ and $6.49$ percentage points, respectively
(\cref{tab:main}). The unlifted control rules out the modified architecture as the cause. Our
ablation studies show that a large part of the increased accuracy is
attributable to the introduced penalty terms.

On the theoretical side, we investigated why the sampled loss is a good surrogate and how sampling regularizes the classification head in~\cref{sec:theory}. 
Moreover, we constructed a one-dimensional model which exhibits the resulting
effect in closed form (\cref{sec:one-dimensional-mechanism}). Together, these results
identify a mechanism by which lifting selects a different, locally more robust
local optimum, rather than optimizing the same objective faster.

The results of our ImageNet investigation suggest a natural extension of our model:
using the empirical covariance, one can detect classes with multiple clusters and introduce
additional class means for them, i.e., model each class as a Gaussian mixture.
This may also require a more expressive $N_2$. We consider this an interesting direction for
future research.

\subsection*{Reproducibility statement}

The software used to produce all results in this work will be made publicly
available in a Git repository upon acceptance of this paper.

\subsection*{AI use statement}
All ideas in this paper, including the lifted formulation, the method, the experimental design, and the theoretical questions addressed in Appendix B, originate from the authors. In Appendix B, generative AI tools were used to formulate the mathematical statements and their proofs and to find a minimal example for the regularization effect.
We have not used generative AI tools for the generation of synthetic datasets, assistance with translation, reformatting datasets, or qualitative and thematic data analysis.
Additionally, we used generative AI tools for refining the language, grammar, and readability of this paper, for formatting and the creation of figures, to discover relevant literature, and to assist in the writing of the software code.
We have reviewed all AI-assisted work. We checked LLM-generated proofs and theorems for correctness. The LLM-generated code was verified and tested for correctness. We take responsibility for the final content of this work, including text, claims or artifacts produced with the aid of generative AI.

\subsection*{Ethics statement}
This is a methodological contribution to supervised image classification. It
involves no human subjects, no personal or otherwise sensitive data, and no
collection of new data. All experiments use established public benchmarks
(CIFAR-10, CIFAR-100, TinyImageNet, and ImageNet) in their standard form and within their
intended research use. We are therefore not aware of any application risk it
introduces beyond those already associated with image classifiers in general.

\section*{Acknowledgments}
The project has been funded by the Deutsche Forschungsgemeinschaft (DFG, German Research Foundation) via grant 314838170, GRK 2297 MathCoRe and the priority program 2331 'Machine Learning in Chemical Engineering' under grant SA 2016/3-1; the European Regional Development Fund under the European Union's Horizon Europe Research and Innovation Program via grant timingMatters ZS/2023/12/182063, grant intelAlgen ZS/2023/12/182064, grant Center for Dynamic Systems ZS/2023/12/182075; and the Innovationsfonds des Gemeinsamen Bundesausschusses (Innovation Committee of the Federal Joint Committee) under grant number KlimaNot (01VSF23017) which we gratefully acknowledge.

\bibliographystyle{styles/iclr2027_conference}
\bibliography{references}
\newpage
\appendix
\section{Additional experiments and technical details}

\subsection{Ablation studies}
We summarize the results for the ablations discussed in \cref{sec:results}.

\begin{table}[h]
  \centering
  \setlength{\tabcolsep}{4pt}
  \renewcommand{\arraystretch}{1.15}
  \small
  \begin{tabular}{llllrrrr}
  \toprule
  Dataset & Model & Variant &  Seed 42 & Seed 43 & Seed 44 & Mean & Spread \\
  \midrule
  CIFAR-10 & ResNet-8 & ProtoSeam          &  $95.27$ & $95.54$ & $95.57$ & $95.46$ & $0.30$ \\
           &          & Unlifted + Penalty &  $95.27$ & $95.50$ & $95.32$ & $95.36$ & $0.23$ \\
  \cmidrule(lr){2-8}
           & ViT-S    & ProtoSeam          &  $91.35$ & $91.64$ & $91.26$ & $91.42$ & $0.38$ \\
           &          & Unlifted + Penalty &  $91.75$ & $92.06$ & $91.02$ & $91.61$ & $1.04$ \\
  \midrule
  CIFAR-100 & ResNet-8 & ProtoSeam          &  $78.18$ & $78.21$ & $78.63$ & $78.34$ & $0.45$ \\
            &          & Unlifted + Penalty &  $78.37$ & $78.21$ & $78.33$ & $78.30$ & $0.16$ \\
  \cmidrule(lr){2-8}
            & ViT-S    & ProtoSeam          &  $70.48$ & $71.00$ & $70.12$ & $70.53$ & $0.88$ \\
            &          & Unlifted + Penalty &  $71.25$ & $71.88$ & $70.05$ & $71.06$ & $1.83$ \\
  \midrule
  TinyImageNet & ResNet-8 & ProtoSeam          &  $65.64$ & $65.48$ & $65.69$ & $65.60$ & $0.21$ \\
               &          & Unlifted + Penalty &  $65.20$ & $64.88$ & $65.21$ & $65.10$ & $0.33$ \\
  \cmidrule(lr){2-8}
               & ViT-S    & ProtoSeam          & $53.95$ & $53.65$ & $54.01$ & $53.87$ & $0.36$ \\
               &          & Unlifted + Penalty & $51.55$ & $51.80$ & $51.81$ & $51.72$ & $0.26$ \\
  \bottomrule
  \end{tabular}
  \caption{Test accuracy (\%) of ProtoSeam and of the unlifted variant with the same penalty terms (Unlifted + Penalty) for seeds 42/43/44, with mean and spread (max $-$ min).}
  \label{tab:unlifted_penalty}
\end{table}

\begin{table}[h]
    \centering        
  \setlength{\tabcolsep}{4pt}
  \renewcommand{\arraystretch}{1.15}
  \small                     
  \begin{tabular}{lllrrrrr}                                 
  \toprule                                                   
  Dataset & Model & Covariance & Seed 42 & Seed 43 & Seed 44 & Mean & Spread \\      
  \midrule                 
  CIFAR-10  & ResNet-8 & empirical      & $95.27$ & $95.54$ & $95.57$ & $95.46$ & $0.30$ \\
            &          & $I/16$ & $95.36$ & $95.45$ & $95.54$ & $95.45$ & $0.18$ \\    
  \cmidrule(lr){2-8}                                                                           
            & ViT-S      & empirical        & $91.35$ & $91.64$ & $91.26$ & $91.42$ & $0.38$ \\
            &          & $I/16$ & $92.22$ & $92.06$ & $92.29$ & $92.19$ & $0.23$ \\                                             
  \midrule                                                                                     
  CIFAR-100 & ResNet-8 & empirical        & $78.18$ & $78.21$ & $78.63$ & $78.34$ & $0.45$ \\                                                
            &          & $I/16$  & $78.03$ & $77.90$ & $78.51$ & $78.15$ & $0.61$ \\                                               
  \cmidrule(lr){2-8}                                                                           
            & ViT-S      & empirical        & $70.48$ & $71.00$ & $70.12$ & $70.53$ & $0.88$ \\
            &          & $I/16$ & $70.66$ & $70.81$ & $70.08$ & $70.52$ & $0.73$ \\                                                
  \midrule                                                                                     
  TinyImageNet & ResNet-8 & empirical   & $65.64$ & $65.48$ & $65.69$ & $65.60$ & $0.21$ \\                                                
            &          & $I/16$         & $65.30$ & $65.68$ & $65.23$ & $65.40$ & $0.45$ \\                                               
  \cmidrule(lr){2-8}
            & ViT-S      & empirical        & $53.95$ & $53.65$ & $54.01$ & $53.87$ & $0.36$ \\
            &          & $I/16$ & $54.35$ & $53.37$ & $53.20$ & $53.64$ & $1.15$ \\        
  \bottomrule                                                                                                                                                   
  \end{tabular}  
  \caption{Test accuracy (\%) of ProtoSeam with the empirical covariance and with the isotropic covariance $I/16$, for seeds 42/43/44, with mean and spread (max $-$ min). Both variants perform almost identically; the isotropic covariance is clearly better only for ViT-S on CIFAR-10.}
    \label{tab:CovarianceAblation}
\end{table}

\paragraph{Unlifted with latent noise.}
  To test whether the gain of lifting can be reproduced by noise alone, we trained the unlifted
  network with isotropic Gaussian noise injected between $N_1$ and $N_2$: during training $N_2$
  receives $N_1(x)+\sigma\varepsilon$ with $\varepsilon\sim\mathcal N(0,I_k)$ drawn independently
  per sample and step, while evaluation uses the noise-free $N_2(N_1(x))$. Unlike in the lifted
  variant, the noise is centered on the embedding $N_1(x)$ rather than on a class anchor, and there
  are no anchors and no penalty terms. On CIFAR-100 with ResNet-8 ($k{=}32$), we swept
  $\sigma\in\{\tfrac{1}{16},\tfrac18,\tfrac14,\tfrac12,1,2,4,8\}$ while keeping learning rate
  and weight decay at the optimal values determined for the unlifted variant. The results are
  summarized in \cref{tab:noise_ablation}.

  This kind of regularization yields little to no benefit.
  The network partly compensates by enlarging its embeddings ($\lVert N_1(x)\rVert$ grows from about
  $18$ to $99$ as $\sigma$ increases from $0$ to $8$). Latent noise alone therefore does not account
  for the gain of lifting.
  Moreover, this regularization differs in kind from that of ProtoSeam: instead of sampling from a class-conditional Gaussian with the estimated class covariance, it only perturbs every sample individually.

\begin{table}
  \centering
  \setlength{\tabcolsep}{4pt}
  \renewcommand{\arraystretch}{1.15}
  \small
  \begin{tabular}{lrrrrrrrrr}
  \toprule
  $\sigma$ & $0$ & $1/16$ & $1/8$ & $1/4$ & $1/2$ & $1$ & $2$ & $4$ & $8$ \\
  \midrule
  Val (\%)        & $76.10$ & $75.80$ & $75.78$ & $76.38$ & $76.06$ & $75.78$ & $75.80$ & $74.82$ & $71.34$ \\
  Test (\%)       & $75.61$ & $75.24$ & $75.52$ & $75.23$ & $75.46$ & $75.23$ & $75.93$ & $74.34$ & $71.31$ \\
  $\|N_1(x)\|$    & $18.1$  & $18.2$  & $18.5$  & $19.9$  & $23.7$  & $31.4$  & $44.3$  & $65.6$  & $99.0$  \\
  \bottomrule
  \end{tabular}
  \caption{Unlifted variant with Gaussian noise of standard deviation $\sigma$ added at the interface (CIFAR-100, ResNet-8). $\sigma = 0$ corresponds to the plain unlifted variant.}
  \label{tab:noise_ablation}
\end{table}

\subsection{Additional benchmarks and visualizations}

\paragraph{Visualizations.}
We show the training process of ViT-S on TinyImageNet in \cref{fig:tin-vit-training-curves}. \Cref{fig:class_centers_mnist} illustrates the prototype locations during training.
\begin{figure}[h]
\centering
\begin{tikzpicture}
\begin{axis}[
    xlabel={Epoch},
    ylabel={Validation accuracy (\%)},
    xmin=0, xmax=200,
    ymin=0, ymax=60,
    grid=major,
    width=0.75\textwidth,
    height=6.cm,
    legend pos=south east,
    legend cell align=left,
]
\addplot[name path=upperbase, draw=none, forget plot] table {%
x y
5 15.550
10 23.320
15 28.350
20 25.700
25 28.910
30 27.850
35 27.680
40 34.340
45 32.730
50 33.070
55 33.490
60 36.920
65 37.470
70 37.500
75 35.200
80 35.790
85 39.240
90 39.320
95 39.490
100 40.180
105 39.930
110 41.240
115 40.930
120 41.230
125 42.110
130 43.660
135 44.330
140 44.220
145 44.680
150 45.200
155 45.470
160 46.540
165 46.650
170 47.220
175 47.040
180 47.930
185 47.980
190 48.300
195 48.510
200 48.620
};
\addplot[name path=lowerbase, draw=none, forget plot] table {%
x y
5 15.200
10 20.650
15 25.430
20 22.270
25 25.620
30 24.250
35 26.670
40 30.160
45 28.530
50 31.150
55 31.420
60 29.790
65 31.050
70 34.290
75 31.530
80 33.260
85 36.160
90 35.690
95 37.030
100 34.590
105 38.510
110 38.710
115 38.660
120 40.270
125 40.670
130 40.630
135 42.790
140 42.990
145 43.150
150 44.110
155 45.070
160 45.210
165 45.570
170 45.660
175 46.500
180 46.630
185 46.900
190 47.130
195 46.950
200 47.050
};
\addplot[gray, opacity=0.18, forget plot] fill between[of=upperbase and lowerbase];

\addplot[name path=upperunl, draw=none, forget plot] table {%
x y
5 14.130
10 20.260
15 24.550
20 25.750
25 26.330
30 28.890
35 28.850
40 25.940
45 28.950
50 31.330
55 30.260
60 29.270
65 33.600
70 33.310
75 31.850
80 33.250
85 34.980
90 35.070
95 34.030
100 34.510
105 37.100
110 37.990
115 38.110
120 39.100
125 41.060
130 40.860
135 41.120
140 42.930
145 43.650
150 44.470
155 44.670
160 45.500
165 45.900
170 46.980
175 47.640
180 47.860
185 48.000
190 48.620
195 49.000
200 48.860
};
\addplot[name path=lowerunl, draw=none, forget plot] table {%
x y
5 12.500
10 18.330
15 20.250
20 23.940
25 20.420
30 22.580
35 23.470
40 24.250
45 21.700
50 24.630
55 23.420
60 25.690
65 26.490
70 28.030
75 27.420
80 30.440
85 28.740
90 32.760
95 30.900
100 32.470
105 33.180
110 33.960
115 36.430
120 36.960
125 36.840
130 40.460
135 39.040
140 40.270
145 41.030
150 42.760
155 43.830
160 44.600
165 45.390
170 45.710
175 45.800
180 46.350
185 45.950
190 46.540
195 46.640
200 46.760
};
\addplot[orange!80!black, opacity=0.18, forget plot] fill between[of=upperunl and lowerunl];

\addplot[name path=upperlif, draw=none, forget plot] table {%
x y
5 8.220
10 11.910
15 14.760
20 14.700
25 19.180
30 20.460
35 23.150
40 24.150
45 25.160
50 28.520
55 27.630
60 30.570
65 30.780
70 31.540
75 34.450
80 37.310
85 35.580
90 35.720
95 37.340
100 41.800
105 38.710
110 41.490
115 43.780
120 44.480
125 45.240
130 45.540
135 46.920
140 48.170
145 49.020
150 49.880
155 50.690
160 50.990
165 52.150
170 52.710
175 52.780
180 53.680
185 53.890
190 54.260
195 54.050
200 54.150
};
\addplot[name path=lowerlif, draw=none, forget plot] table {%
x y
5 7.360
10 11.630
15 9.800
20 13.670
25 16.440
30 19.930
35 21.450
40 22.220
45 23.500
50 23.560
55 21.810
60 28.600
65 28.900
70 30.740
75 31.120
80 30.940
85 34.220
90 30.800
95 35.680
100 38.870
105 38.140
110 39.670
115 35.450
120 42.980
125 42.190
130 44.690
135 45.970
140 46.780
145 47.940
150 49.240
155 50.440
160 50.360
165 51.380
170 51.900
175 52.460
180 52.770
185 53.010
190 53.580
195 53.760
200 53.560
};
\addplot[blue!70!black, opacity=0.18, forget plot] fill between[of=upperlif and lowerlif];

\addplot[color=gray, thick, mark=none] table {%
x y
5 15.377
10 22.040
15 27.060
20 23.473
25 27.603
30 26.383
35 27.310
40 32.287
45 30.413
50 32.063
55 32.540
60 33.633
65 33.833
70 36.143
75 33.937
80 34.750
85 37.587
90 37.180
95 38.200
100 37.147
105 38.993
110 39.970
115 40.137
120 40.713
125 41.263
130 42.067
135 43.537
140 43.573
145 43.933
150 44.573
155 45.240
160 45.780
165 45.953
170 46.637
175 46.727
180 47.150
185 47.307
190 47.567
195 47.593
200 47.707
};
\addlegendentry{Baseline}

\addplot[color=orange!80!black, thick, mark=none] table {%
x y
5 13.493
10 19.263
15 22.637
20 24.623
25 22.983
30 25.213
35 25.533
40 25.247
45 26.083
50 27.090
55 27.200
60 27.537
65 30.983
70 31.217
75 29.297
80 32.267
85 31.803
90 34.147
95 32.630
100 33.813
105 35.020
110 35.460
115 37.513
120 37.967
125 38.643
130 40.720
135 39.927
140 41.557
145 42.433
150 43.353
155 44.253
160 44.913
165 45.567
170 46.227
175 46.800
180 47.073
185 47.150
190 47.713
195 47.750
200 47.777
};
\addlegendentry{Unlifted}

\addplot[color=blue!70!black, thick, mark=none] table {%
x y
5 7.690
10 11.793
15 12.450
20 14.283
25 17.467
30 20.233
35 22.320
40 23.103
45 24.337
50 25.853
55 25.497
60 29.563
65 29.970
70 31.217
75 33.200
80 34.920
85 35.113
90 33.600
95 36.317
100 39.873
105 38.360
110 40.343
115 40.650
120 43.827
125 43.663
130 45.230
135 46.347
140 47.577
145 48.470
150 49.590
155 50.570
160 50.730
165 51.650
170 52.227
175 52.640
180 53.250
185 53.360
190 53.973
195 53.923
200 53.877
};
\addlegendentry{ProtoSeam}

\end{axis}
\end{tikzpicture}
\caption{ViT-S on TinyImageNet: validation accuracy over the epochs for
baseline, unlifted, and lifted variant. We show the mean
over $3$ seeds ($42,43,44$), sampled every $5$ epochs. The shaded bands show
the full seed range (min to max) at each sampled epoch.}
\label{fig:tin-vit-training-curves}
\end{figure}

\begin{figure}[h]
    \centering
    \includegraphics[width=0.6\linewidth]{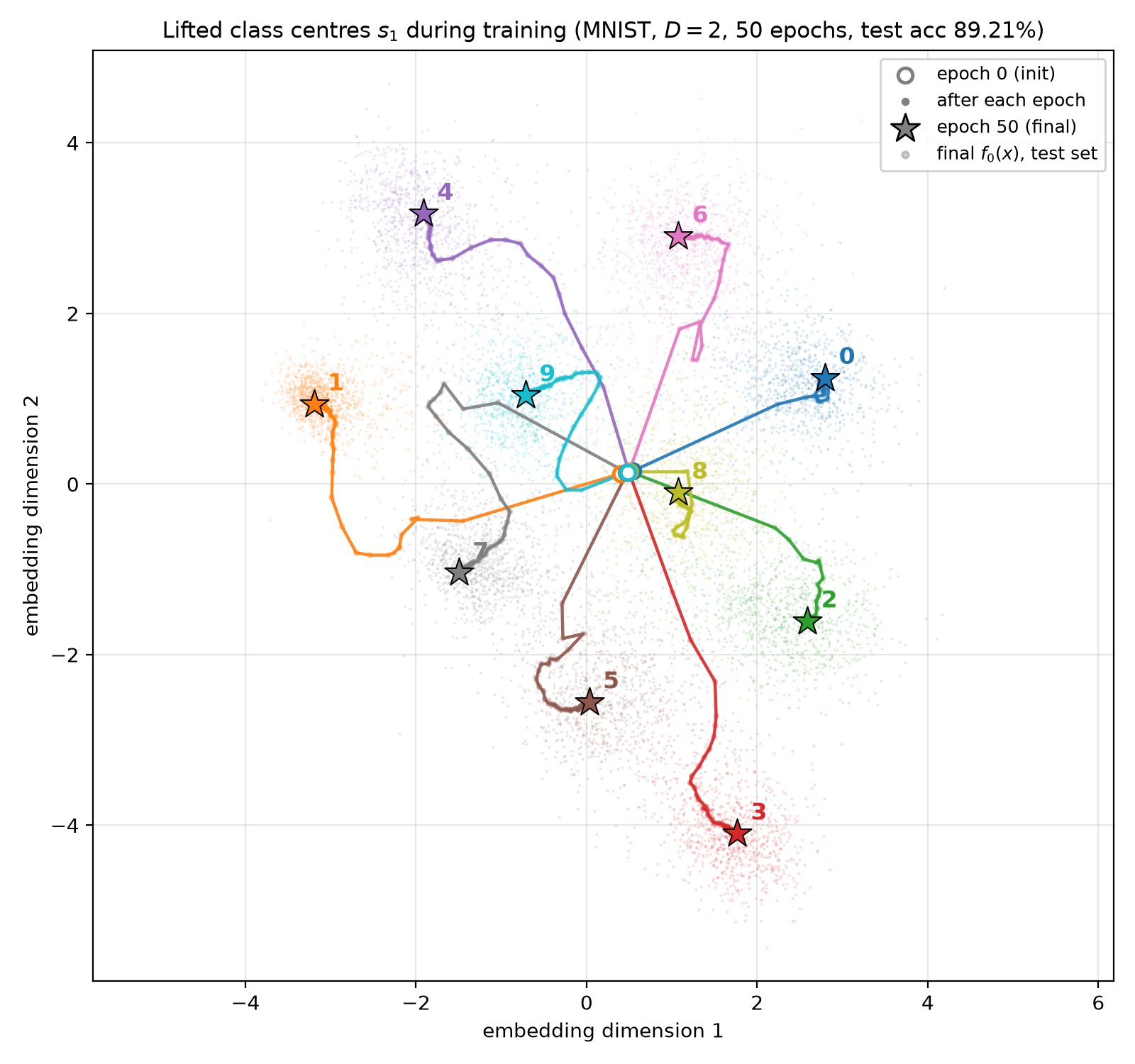}
    \caption{Illustration of the prototype positions over the course of 50 training epochs on MNIST for the network used in \cref{fig:cluster}. We use FSInit and observe that the clusters start very close to each other, owing to the random initialization of the network weights. The clusters then drift apart while keeping their semantic arrangement intact (e.g., $6$ close to $0$, or $3$ next to $5$).}
    \label{fig:class_centers_mnist}
\end{figure}

\paragraph{Robustness to common corruptions.}
We evaluate the networks trained on CIFAR-10, each with its best hyperparameters from \cref{tab:hyperparameters}, on CIFAR-10-C~\citep{hendrycks2019benchmarking}.
Following the standard protocol, we report C-15, the mean accuracy over the 15 benchmark corruptions (noise: Gaussian, shot, impulse; blur: defocus, glass, motion, zoom; weather: snow, frost, fog, brightness; digital: contrast, elastic, pixelate, JPEG) and all five severity levels.
The four additional corruptions (speckle noise, Gaussian blur, spatter, saturate) are averaged separately as C-extra.
We illustrate the performance across the different severity levels in \cref{fig:cifar_10_c} and summarize all further quantities in \cref{tab:cifar10c}.
As \cref{tab:cifar10c} shows, ProtoSeam attains the highest accuracy for both backbones on clean data, on C-15, at the lowest and highest severity, and on C-extra, and it retains the largest fraction of its clean accuracy under corruption.

\begin{figure}
    \centering
    \includegraphics[width=0.9\linewidth]{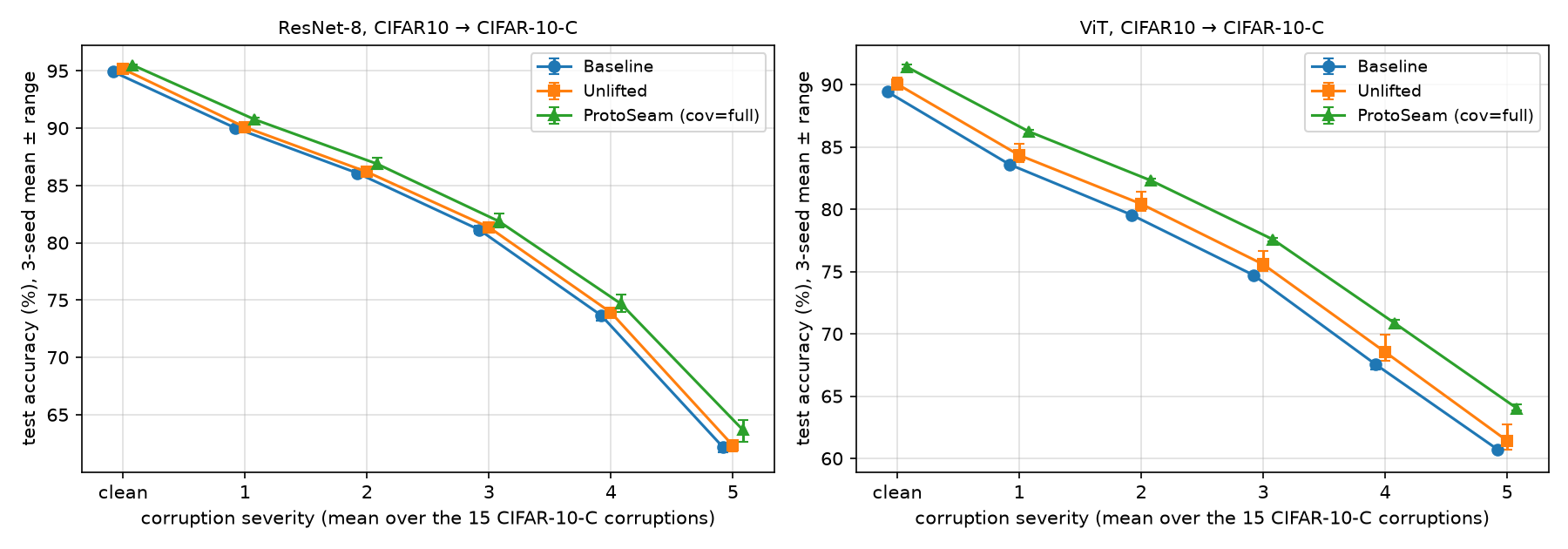}
    \caption{Test accuracy of ResNet-8 and ViT-S on CIFAR-10-C over the five corruption severity levels. ProtoSeam performs best and maintains its lead over the other variants at every severity.}
    \label{fig:cifar_10_c}
\end{figure}

\begin{table}[h]
  \centering
  \small
  \setlength{\tabcolsep}{4pt}
  \renewcommand{\arraystretch}{1.15}
  \begin{tabular}{llrrrrrr}
  \toprule
  Model & Variant & Clean & C-15 & Sev.~1 & Sev.~5 & C-extra & rel. \\
  \midrule
  ResNet-8 & Baseline  & $94.91$ ($0.32$) & $78.61$ ($0.56$) & $90.01$ & $62.15$ & $82.00$ & $0.828$ \\
           & Unlifted  & $95.16$ ($0.38$) & $78.76$ ($0.47$) & $90.08$ & $62.30$ & $82.18$ & $0.828$ \\
           & ProtoSeam & $\mathbf{95.46}$ ($0.30$) & $\mathbf{79.58}$ ($1.22$) & $\mathbf{90.75}$ & $\mathbf{63.68}$ & $\mathbf{83.20}$ & $\mathbf{0.834}$ \\
  \midrule
  ViT-S    & Baseline  & $89.43$ ($0.16$) & $73.23$ ($0.34$) & $83.60$ & $60.73$ & $76.21$ & $0.819$ \\
           & Unlifted  & $90.05$ ($0.89$) & $74.06$ ($1.72$) & $84.34$ & $61.42$ & $76.98$ & $0.822$ \\
           & ProtoSeam & $\mathbf{91.42}$ ($0.38$) & $\mathbf{76.20}$ ($0.41$) & $\mathbf{86.24}$ & $\mathbf{64.02}$ & $\mathbf{79.03}$ & $\mathbf{0.834}$ \\
  \bottomrule
  \end{tabular}
  \caption{Accuracy (\%) on CIFAR-10 (clean) and CIFAR-10-C for networks trained on CIFAR-10, as the mean over three seeds with the spread (max $-$ min) in parentheses. Sev.~1 and Sev.~5 average the 15 benchmark corruptions at the lowest and highest severity. The column rel.\ is the ratio of C-15 to clean accuracy, i.e., the fraction of clean accuracy retained under corruption.}
  \label{tab:cifar10c}
\end{table}

\paragraph{Class separation in the latent space.}
  To quantify how well the classes are separated in the latent space, we evaluated the trained
  ResNet-8 models ($k{=}32$) on the $M{=}10{,}000$ CIFAR-100 test images. Let
  $z_i=N_1(x_i)\in\mathbb R^k$ denote the latent representation of image $i$ with label $y_i\in\{1,\dots,C\}$, $C{=}100$.
  Let $n_c$ be the number of images and $\mu_c$ the mean latent of class $c$, and let $\bar z$ be the
  global mean. We use five measures, all invariant under a global rescaling $z\mapsto az$:
  \begin{itemize}
  \item \textbf{Fisher ratio:} $\operatorname{tr}S_B/\operatorname{tr}S_W$, where
    $$
    S_B=\sum_c n_c(\mu_c-\bar z)(\mu_c-\bar z)^\top \quad \text{ and } \quad
    S_W=\sum_i(z_i-\mu_{y_i})(z_i-\mu_{y_i})^\top$$ 
    are the between-class and within-class scatter
    matrices (higher is better).
  \item \textbf{Mean silhouette:} $\frac1M\sum_i (b_i-a_i)/\max(a_i,b_i)$, where $a_i$ is the mean
    Euclidean distance of $z_i$ to the other points of its own class and $b_i$ is the smallest mean
    distance of $z_i$ to the points of any other class; the silhouette lies in $[-1,1]$ (higher is
    better).
  \item \textbf{Davies--Bouldin index:} $\frac1C\sum_c\max_{c'\neq c}(s_c+s_{c'})/\lVert\mu_c-\mu_{c'}\rVert$,
    where $s_c$ is the mean distance of the points of class $c$ to $\mu_c$ (lower is better).
  \item \textbf{Nearest-neighbor ratio:} $\frac1C\sum_c\min_{c'\neq c}\lVert\mu_c-\mu_{c'}\rVert/r_c$, where
    $r_c$ is the root-mean-square distance of the points of class $c$ to $\mu_c$. A value below $1$
    means that a class overlaps its nearest neighboring class (higher is better).
  \item \textbf{Nearest-class-mean error:} the fraction of images with
    $\arg\min_c\lVert z_i-\mu_c\rVert\neq y_i$ (lower is better).
  \end{itemize}
  For the end-to-end baseline, which has no bottleneck, $z_i$ is the $512$-dimensional input of the
  classifier. Each model is evaluated at its selected hyperparameters with three seeds.
  We show the results in \cref{tab:cluster_separation}.
  ProtoSeam separates the classes clearly better than both the unlifted and the baseline model on
  the first four measures, and it attains the lowest nearest-class-mean error.

  \begin{table}
  \centering
  \setlength{\tabcolsep}{4pt}
  \renewcommand{\arraystretch}{1.15}
  \small
  \begin{tabular}{lrrrrr}
  \toprule
  Variant & Fisher ratio $\uparrow$ & Silhouette $\uparrow$ & Davies--Bouldin $\downarrow$ & NN ratio $\uparrow$ & NCM error $\downarrow$ \\
  \midrule
  ProtoSeam       & $\mathbf{3.71}$ $(0.03)$ & $\mathbf{0.272}$ $(0.003)$ & $\mathbf{1.311}$ $(0.017)$ & $\mathbf{1.570}$ $(0.004)$ & $\mathbf{0.227}$ $(0.001)$ \\
  Unlifted   & $1.27$ $(0.02)$          & $0.073$ $(0.001)$          & $2.703$ $(0.013)$          & $0.795$ $(0.004)$          & $0.260$ $(0.003)$ \\
  Baseline & $1.17$ $(0.01)$          & $0.115$ $(0.001)$          & $2.284$ $(0.015)$          & $0.947$ $(0.004)$          & $0.232$ $(0.005)$ \\
  \bottomrule
  \end{tabular}
  \caption{Class-cluster separation of the learned representation on the CIFAR-100 test set (ResNet-8). Mean over 3 seeds, with the spread (max $-$ min) in parentheses.}
  \label{tab:cluster_separation}
\end{table}

\subsection{Technical details}
\label{sec:tech_details}
\paragraph{Training.}
For ResNet-8 and ViT-S, for each of the three datasets CIFAR-10, CIFAR-100, and TinyImageNet, and for every variant, we performed a hyperparameter grid search across the learning rates $\{0.01, 0.02, 0.05, 0.1\}$ and weight decays $\{10^{-4}, 2{\times}10^{-4}, 5{\times}10^{-4}, 10^{-3}, 2{\times}10^{-3}\}$. 
We always picked the combination with the highest validation accuracy after a full run of 200 epochs.
We chose 200 epochs to avoid underfitting; increasing the training length to 400 epochs did not improve accuracy.
The determined best hyperparameter combinations are summarized in Table~\ref{tab:hyperparameters}.
All split variants (Unlifted, Unlifted + Penalty, and ProtoSeam) use the embedding dimension $k=32$; all variants with penalty terms use $\rho_{\min} = 1$, $\rho_{\max} = 16$, and $\alpha = 2$. As regularization of the sampling covariance, we use $\sigma_0=10^{-6}$ by default. To prevent the Cholesky decomposition from failing, a safeguard increases the regularization further if necessary; this happened only in runs with poor hyperparameters.

\begin{table}[htbp]
  \centering
  \setlength{\tabcolsep}{6pt}
  \renewcommand{\arraystretch}{1.2}
  \begin{tabular}{llllr}
  \toprule
  Dataset & Model & Variant & $\eta$ & $\lambda$ \\
  \midrule
  CIFAR-10     & ResNet-8 & Baseline           & $0.02$  & $10^{-3}$          \\
               &         & Unlifted           & $0.02$  & $10^{-3}$          \\
               &         & Unlifted + Penalty  & $0.05$  & $5{\times}10^{-4}$ \\
               &         & ProtoSeam (cov=$I/16$)  & $0.05$  & $5{\times}10^{-4}$ \\
               &         & ProtoSeam  & $0.05$  & $5{\times}10^{-4}$ \\
  \cmidrule(lr){2-5}
               & ViT-S   & Baseline           & $0.02$  & $10^{-4}$          \\
               &         & Unlifted           & $0.02$  & $10^{-4}$          \\
               &         & Unlifted + Penalty  & $0.01$  & $2{\times}10^{-4}$ \\
               &         & ProtoSeam (cov=$I/16$)  & $0.02$  & $10^{-4}$ \\
               &         & ProtoSeam  & $0.01$  & $10^{-4}$ \\
  \midrule
  CIFAR-100    & ResNet-8 & Baseline           & $0.05$  & $2{\times}10^{-3}$ \\
               &         & Unlifted           & $0.1$  & $5{\times}10^{-4}$  \\
               &         & Unlifted + Penalty  & $0.02$  & $2{\times}10^{-3}$ \\
               &         & ProtoSeam (cov=$I/16$)  & $0.02$  & $2{\times}10^{-3}$ \\
               &         & ProtoSeam  & $0.02$  & $2{\times}10^{-3}$ \\
  \cmidrule(lr){2-5}
               & ViT-S   & Baseline           & $0.01$  & $10^{-4}$          \\
               &         & Unlifted           & $0.1$  & $10^{-4}$          \\
               &         & Unlifted + Penalty  & $0.02$  & $5{\times}10^{-4}$ \\
               &         & ProtoSeam (cov=$I/16$)  & $0.01$  & $2{\times}10^{-4}$ \\
               &         & ProtoSeam  & $0.01$  & $5{\times}10^{-4}$ \\
  \midrule
  TinyImageNet & ResNet-8 & Baseline           & $0.05$  & $10^{-3}$          \\
               &         & Unlifted           & $0.1$  & $2{\times}10^{-4}$          \\
               &         & Unlifted + Penalty  & $0.02$  & $10^{-3}$ \\
               &         & ProtoSeam (cov=$I/16$)  & $0.02$  & $2{\times}10^{-3}$ \\
               &         & ProtoSeam  & $0.02$  & $2{\times}10^{-3}$ \\
  \cmidrule(lr){2-5}
               & ViT-S   & Baseline           & $0.1$   & $2{\times}10^{-4}$ \\
               &         & Unlifted           & $0.1$  & $2{\times}10^{-4}$   \\
               &         & Unlifted + Penalty  & $0.01$  & $2{\times}10^{-4}$ \\
               &         & ProtoSeam (cov=$I/16$)  & $0.02$  & $2{\times}10^{-4}$ \\
               &         & ProtoSeam  & $0.02$  & $2{\times}10^{-4}$ \\
  \bottomrule
  \end{tabular}
  \caption{Best hyperparameters for every
  (dataset, model, variant) combination, determined by the grid search over learning rate $\eta \in \{0.01, 0.02, 0.05, 0.1\}$ and weight decay $\lambda \in \{10^{-4}, 2{\times}10^{-4}, 5{\times}10^{-4}, 10^{-3}, 2{\times}10^{-3}\}$. 
  }
  \label{tab:hyperparameters}
  \end{table}

\paragraph{Influence of $k$ on the unlifted variant.} We observed different patterns for ResNet-8 and ViT-S. For ResNet-8, there is a clear trend of improved performance when increasing $k$
 up to the number of classes, where it essentially matches the baseline. Increasing $k$
 further leads only to very minor improvements (as seen for CIFAR-10). The ViT-S, on the other hand, seems to benefit from this bottleneck. Here, performance drops to the baseline level, or slightly above it, when $k$ is increased to the number of classes and beyond.

\paragraph{Vision Transformer.} For ViT-S we use the SimpleViT architecture of \citet{beyer2022better} as implemented in vit-pytorch \citep{wang2025vitpytorch}, with patch size 4, dimension 512, depth 6, 8 heads, and MLP dimension 512.

\paragraph{Datasets.}
All three datasets (CIFAR-10, CIFAR-100, TinyImageNet) use the same training-time augmentation pipeline: a random crop with 4-pixel zero-padding back to the native resolution ($32\times32$ for CIFAR-10/100, $64\times64$ for TinyImageNet), a random horizontal flip, RandAugment (default policy), normalization with each dataset's own per-channel mean and standard deviation, and finally random erasing (probability $0.5$, area fraction $0.02$--$0.33$). We held out $10\%$ of the training set as the validation set. As test sets, we use the official test sets for CIFAR-10 and CIFAR-100 and the official validation set for TinyImageNet.
Validation and test images are only normalized, without stochastic augmentation, so the reported accuracies do not depend on augmentation randomness.

\paragraph{ImageNet.}
  We use ImageNet-1k (ILSVRC-2012, 1000 classes), with images taken from a copy whose shorter
  side was resized to 256 pixels. Since all methods use these images, the comparisons
  are unaffected, though absolute accuracies may be slightly below those obtained with the original
  images. The official validation set ($50{,}000$ images, 50 per class) serves
  as the test set and is never used for model selection. From the $1{,}281{,}167$ training images we
  hold out a fixed random subset of $25{,}000$ images (drawn once with seed $0$, identical for all
  runs) as a validation set for hyperparameter selection. The remaining $1{,}256{,}167$ images are
  used for training. During training we apply random resized cropping to $224{\times}224$ (scale
  $[0.08,1]$, aspect ratio $[3/4,4/3]$) and random horizontal flips with probability $0.5$. For
  validation and testing, images are resized to a shorter side of 256 and center-cropped to
  $224{\times}224$. All images are normalized with the standard ImageNet channel means and standard
  deviations. All methods use the same data, splits, and augmentation.

 \begin{figure}[h]
    \centering
    \includegraphics[width=0.6\linewidth]{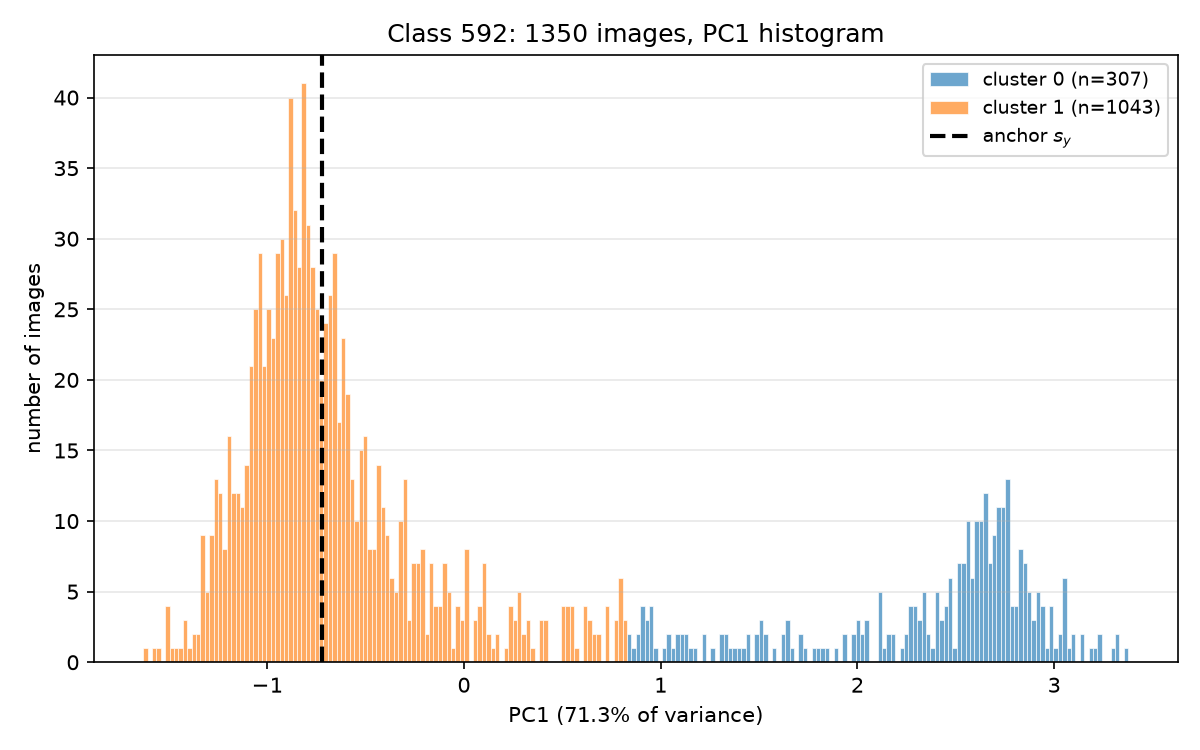}
    \caption{Visualization of the latent distribution of $N_1(x)$ for a lifted ResNet-34 trained on ImageNet, where $x$ is a sample of ImageNet class 592 (hard disk). We observe two clearly distinguishable clusters. To obtain more samples for the visualization, the plot includes all training, validation, and test images of this class.}
    \label{fig:imagenet_hard_disks}
\end{figure}

 \begin{figure}[h]
    \centering
    \includegraphics[width=0.2\linewidth]{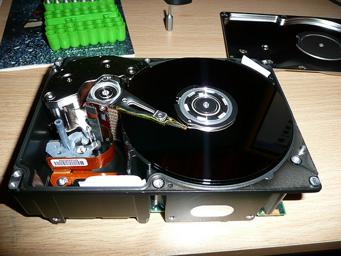}
    \includegraphics[width=0.2\linewidth]{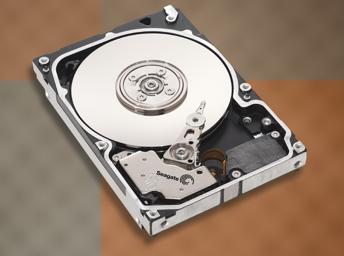} \hspace*{0.15\linewidth}
    \includegraphics[width=0.2\linewidth]{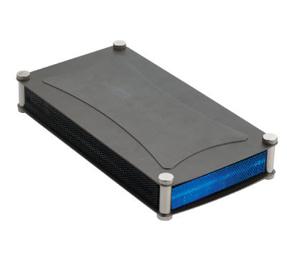}
    \includegraphics[width=0.2\linewidth]{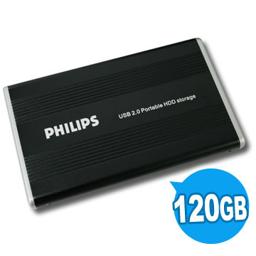}
    \caption{Exemplary images for the clusters of \cref{fig:imagenet_hard_disks}. \textbf{Left:} Two images from cluster 1, showing hard disks built into computers. \textbf{Right:} Two images from cluster 0, showing portable hard disks.}
    \label{fig:imagenet_hard_disks_samples}
\end{figure}

\paragraph{GaussianImageNet.}
  GaussianImageNet is the subset of ImageNet-1k (ILSVRC-2012) consisting of the classes whose learned
  embeddings show no well-separated secondary cluster. To select the classes we used the
  embedding network $N_1$ of a lifted ResNet-34 ($k{=}64$, 150 epochs) trained on the full ImageNet-1k.
  For every class $c$ we drew 150 training images at random (center crop $224{\times}224$), computed
  the distances 
  $$
  d_i = \lVert N_1(x_i) - \bar{N}_1^{(c)} \rVert_2
  $$
  to the class mean $\bar{N}_1^{(c)}$, and normalized them by their median. We then fitted one- and two-component
  Gaussian mixtures to these one-dimensional distances. A class was labeled as a genuine mixture,
  i.e., as containing a second cluster at a markedly different distance from the class prototype,
  if the minority component had weight at least $0.15$ and the two component means were separated
  by at least two pooled standard deviations. This applied to 706 of the 1000 classes. The remaining
  294 classes form GaussianImageNet. Only training images entered the selection; the official
  validation set was not used. GaussianImageNet keeps all images of the selected classes
  ($377{,}968$ training and $14{,}700$ validation images). We use the 50 official validation
  images per class ($14{,}700$ in total) as the test set. A fixed random subset of $25{,}000$
  images (drawn with seed $0$) is held out from the training images as a validation set for
  hyperparameter selection, which leaves $352{,}968$ images for training. Images, preprocessing, and
 augmentation are identical to our ImageNet-1k setup, so the two datasets differ only in the class set.
\section{Why Lifting Improves Accuracy}
\label{sec:theory}

The gains of Table~\ref{tab:main} call for an explanation.  Two mechanisms could be
responsible.  The first is optimization: the lifted objective may simply be
easier to minimize.  We do not rule this mechanism out, but it is not the one
analysed here.  The second mechanism is
statistical: the lifted objective may select a \emph{different} minimizer.
This is the mechanism which we want to study here. Our theoretical investigation
is split into three parts which cover the following aspects.
\begin{enumerate}
    \item[\cref{sec:risk-transfer}] \textbf{The sampled loss is a good surrogate.} For ProtoSeam the head $N_2$ is trained on Gaussian samples, but deployed on real embeddings. We show that the training loss of the head is equal to the risk of the reconnected network up to an explicit mismatch term. Moreover, we show that every part of this term is bounded by a quantity which the method controls (Theorem~\ref{thm:risk-transfer} and Corollary~\ref{cor:controlled}).
    \item[\cref{sec:sensitivity}] \textbf{Sampling regularizes the head.} Compared with evaluating the head only at the prototypes, sampling adds a non-negative penalty. To second order this is a penalty on the curvature of the head loss (Proposition~\ref{prop:smoothing}). In the regime of small perturbations, among heads with equal loss at the prototypes, the sampled objective prefers the one whose loss curves least around them.
    \item[\cref{sec:one-dimensional-mechanism}] \textbf{A minimal example.} We demonstrate the regularization effect in closed form.
\end{enumerate}

\begin{table}[h]
  \centering
  \small
  \renewcommand{\arraystretch}{1.25}
  \begin{tabular}{p{0.2\linewidth}p{0.5\linewidth}}
  \toprule
  Notation & Description \\
  \midrule
  $\ell_i(z)$   & the loss of the head $N_2$ on an embedding $z$ whose true class is $i$. \\
  $P_i$         & the actual distribution of the embeddings $N_1(X)$ of class $i$, with mean $\mu_i$ and covariance $\Sigma_i$. \\
  $\Gamma_i=\mathcal{N}(\mu_i,\Sigma_i)$ & the Gaussian with the same mean and covariance as $P_i$. \\
  $Q_i=\mathcal{N}(s_i,\tilde\Sigma_i)$  & the distribution from which the training inputs of $N_2$ are sampled, with $\tilde\Sigma_i=\hat\Sigma_i+\sigma_0^2I_k$. \\
  $R$           & the risk of the deployed network (average of $\ell_i$ over $P_i$). \\
  $\Lsmp$       & the sampled loss (average of $\ell_i$ over $Q_i$). \\
  $G$           & a Lipschitz constant of the head losses $\ell_i$. \\
  \bottomrule
  \end{tabular}
  \caption[Notation]{Overview of the notation used in the following sections.}
\end{table}

\subsection{The Sampled Loss is the Deployed Risk up to Interface Mismatch}
\label{sec:risk-transfer}

At inference the prototypes are discarded, and the deployed classifier is the
reconnected network $N_2\circ N_1$. However, the head $N_2$ has never seen an embedding $N_1(x)$ during training, only Gaussian samples around the prototypes. This subsection examines what that substitution costs. We write
\begin{equation}
    \ell_i(z) := \Lcls\bigl(N_2(z),\,i\bigr),
    \qquad z\in\R^k,
    \label{eq:head-loss}
\end{equation}
for the loss of the head on an embedding of class $i$. Under balanced classes, the population risk of
the deployed classifier is
\begin{equation}
    R(\theta_1,\theta_2)
      := \E\bigl[\ell_Y\bigl(N_1(X)\bigr)\bigr]
       = \frac1n\sum_{i=1}^n \E_{z\sim P_i}\,\ell_i(z),
    \label{eq:deployed-risk}
\end{equation}
where $P_i$ denotes the law of $N_1(X)$ conditioned on $Y=i$ (an italic $P$,
not to be confused with the repulsion penalty $\mathcal{P}$), with mean
$\mu_i$ and covariance $\Sigma_i$.  The classification term of
\eqref{eq:combinedLoss} has the same form, only with each $P_i$ replaced by the sampling distribution $Q_i$:
\begin{equation}
    \Lsmp
      := \frac1n\sum_{i=1}^n \E_{z\sim Q_i}\,\ell_i(z),
    \qquad
    Q_i := \mathcal{N}\bigl(s_i,\tilde\Sigma_i\bigr).
    \label{eq:sampled-risk}
\end{equation}
The difference between $R$ and $\Lsmp$ therefore lies only in the distribution over which the head loss is averaged, and their gap is governed by how far $Q_i$ is from $P_i$.
Training $N_2$ on \eqref{eq:sampled-risk} is vicinal risk
minimization~\cite{chapelle2000vicinal} at the lifting interface: the
empirical embeddings are replaced by Gaussian vicinities around the
prototypes.  

We measure the distance between the embedding law $P_i$ and the sampling law $Q_i$ in the 1-Wasserstein distance $W_1$ \citep{villani2009optimal}. Colloquially speaking, $W_1(P,Q)$ is the smallest average distance over which probability mass has to be moved to turn $P$ into $Q$. Between covariance matrices we use the Bures distance
\begin{equation}
    \BW(A,B)
      := \Bigl(\operatorname{tr} A + \operatorname{tr} B
         - 2\operatorname{tr}\bigl[\bigl(B^{1/2}AB^{1/2}\bigr)^{1/2}\bigr]\Bigr)^{1/2},
    \label{eq:bures}
\end{equation}
which is a metric on positive semi-definite matrices~\citep{bhatia2019bures}
and is exactly the covariance part of the $2$-Wasserstein distance between
Gaussians~\citep{givens1984class}. For commuting $A$ and $B$ it reduces to $\|A^{1/2} - B^{1/2}\|_F$, so it compares standard deviations rather than variances.

\begin{assumption}
\label{ass:lipschitz}
Each head loss $\ell_i:\R^k\to\R$ is $G$-Lipschitz:
$|\ell_i(z)-\ell_i(z')|\le G\|z-z'\|_2$ for all $z,z'\in\R^k$ and all
$i\in[n]$.
\end{assumption}

\begin{remark}[The assumption is exact for our architecture]
\label{rem:lipschitz-ours}
All experiments in \cref{sec:experiments} use a single linear head,
$N_2(z)=W_2z+b$, with the cross-entropy loss.  Then
$\nabla_z\ell_i(z)=W_2^\top\bigl(p(z)-e_i\bigr)$ with
$p(z)=\operatorname{softmax}(W_2z+b)$, and since
$\|p-e_i\|_2\le\sqrt2$ for any probability vector $p$,
Assumption~\ref{ass:lipschitz} holds globally with $G=\sqrt2\,\|W_2\|_2$.  The weight
decay applied to $N_2$ keeps $\|W_2\|_2$ small and therefore directly controls the constant $G$ in the bound
below.
\end{remark}

\begin{theorem}[Interface risk transfer]
\label{thm:risk-transfer}
Under Assumption~\ref{ass:lipschitz} and balanced classes,
\begin{equation}
    \bigl|R-\Lsmp\bigr|
      \;\le\; \frac{G}{n}\sum_{i=1}^n W_1\bigl(P_i,Q_i\bigr),
    \label{eq:risk-transfer}
\end{equation}
and for every class $i$, if $P_i$ has finite second moment,
\begin{equation}
    W_1\bigl(P_i,Q_i\bigr)
      \;\le\;
      \underbrace{W_1\bigl(P_i,\mathcal{N}(\mu_i,\Sigma_i)\bigr)}
        _{\textnormal{non-Gaussianity }\Delta_i}
      \;+\;
      \Bigl(\underbrace{\|\mu_i-s_i\|_2^2}_{\textnormal{mean mismatch}}
      +\underbrace{\BW^2\bigl(\Sigma_i,\tilde\Sigma_i\bigr)}
        _{\textnormal{covariance mismatch}}\Bigr)^{1/2}.
    \label{eq:mismatch-decomposition}
\end{equation}
\end{theorem}

Before continuing with the proof we want to describe it in words. By \eqref{eq:risk-transfer}, the sampled loss misjudges the deployed risk by at most $G$ times the average distance between what the head is trained on and what it is deployed on. According to \eqref{eq:mismatch-decomposition} this distance has three sources: the shape of $P_i$ (which is not necessarily Gaussian), its location ($\mu$ need not coincide with $s_i$), and its spread ($\tilde\Sigma_i$ need not equal $\Sigma_i$).

\begin{proof}
First, we want to prove \eqref{eq:risk-transfer}. Fix $i$ and let $\gamma$ be any coupling of
$(P_i,Q_i)$, i.e.\ a joint law of $(z,z')$ with marginals $P_i$ and $Q_i$. (If $P_i$ has no finite first moment, then $W_1(P_i,Q_i)=\infty$ and there is nothing to prove.)
Then
\[
    \bigl|\E_{P_i}\ell_i-\E_{Q_i}\ell_i\bigr|
      =\bigl|\E_{(z,z')\sim\gamma}\bigl[\ell_i(z)-\ell_i(z')\bigr]\bigr|
      \le G\,\E_{\gamma}\|z-z'\|_2 ,
\]
by Assumption~\ref{ass:lipschitz}.  Taking the infimum over couplings gives
$|\E_{P_i}\ell_i-\E_{Q_i}\ell_i|\le G\,W_1(P_i,Q_i)$
\citep[Chapter 6]{villani2009optimal}. Averaging over $i$ and using the triangle
inequality for the average yields \eqref{eq:risk-transfer}.

Next, let us address \eqref{eq:mismatch-decomposition}. We pass from $P_i$ to $Q_i$ via the moment-matched Gaussian
$\Gamma_i:=\mathcal{N}(\mu_i,\Sigma_i)$. Since $W_1$ is a metric,
$W_1(P_i,Q_i)\le W_1(P_i,\Gamma_i)+W_1(\Gamma_i,Q_i)$. The first summand is $\Delta_i$ by definition. For the second one, both measures are Gaussian. By Jensen's
inequality $W_1\le W_2$, and for Gaussian measures the $W_2$ distance is
available in closed form~\citep{givens1984class}:
\[
    W_2^2\bigl(\mathcal{N}(\mu_i,\Sigma_i),\mathcal{N}(s_i,\tilde\Sigma_i)\bigr)
      =\|\mu_i-s_i\|_2^2+\BW^2\bigl(\Sigma_i,\tilde\Sigma_i\bigr),
\]
with $\BW$ as in \eqref{eq:bures}.  Combining the triangle inequality, $W_1 \leq W_2$, and this closed form completes the proof.
\end{proof}

The argument is elementary, but its value lies in accounting.
Each term of \eqref{eq:mismatch-decomposition} is bounded by a
quantity that the lifted objective controls, as the next corollary makes
precise.

\begin{corollary}[The objective controls its own transfer gap]
\label{cor:controlled}
Let $\Lcon:=\frac{\rho}{2}\,\E\|N_1(X)-s_Y\|_2^2$ denote the population
consensus term of \eqref{eq:combinedLoss}, let
$\tilde\Sigma_i=\hat\Sigma_i+\sigma_0^2I_k$ as in \eqref{eq:shrinkage}, and
let $\bar\Delta:=\frac1n\sum_i\Delta_i$ be the average non-Gaussianity.  Then,
under the assumptions of Theorem~\ref{thm:risk-transfer},
\begin{equation}
    R \;\le\; \Lsmp
      + G\Bigl[\;\bar\Delta
      \;+\; \sqrt{\tfrac{2}{\rho}\Lcon}
      \;+\; \sqrt{k}\,\sigma_0
      \;+\; \tfrac{1}{2\sigma_0}\cdot\tfrac1n\textstyle\sum_i
            \bigl\|\hat\Sigma_i-\Sigma_i\bigr\|_F\;\Bigr].
    \label{eq:controlled-bound}
\end{equation}
Moreover, the non-Gaussianity is itself controlled by the consensus term,
$\bar\Delta\le2\sqrt{\tfrac{2}{\rho}\Lcon}$, and
\begin{equation}
    R \;\le\; \Lsmp
      + G\Bigl[\;\sqrt{5}\,\sqrt{\tfrac{2}{\rho}\Lcon}
      \;+\; \sqrt{k}\,\sigma_0
      \;+\; \tfrac{1}{2\sigma_0}\cdot\tfrac1n\textstyle\sum_i
            \bigl\|\hat\Sigma_i-\Sigma_i\bigr\|_F\;\Bigr].
    \label{eq:controlled-bound-full}
\end{equation}
\end{corollary}

\begin{proof}
We apply \eqref{eq:risk-transfer}, \eqref{eq:mismatch-decomposition}, and
$\sqrt{a^2+b^2}\le a+b$ (valid for $a,b\ge0$) to each class:
\begin{equation*}
    R-\Lsmp\;\le\;\frac{G}{n}\sum_{i=1}^n
      \Bigl[\Delta_i+\|\mu_i-s_i\|_2+\BW\bigl(\Sigma_i,\tilde\Sigma_i\bigr)\Bigr]. \label{eq:step0_inequality}
\end{equation*}
It remains to bound the mean, non-Gaussianity, and covariance terms individually.

\emph{Step 1: Mean term.}  By Jensen's inequality applied twice, first inside the norm and then to the average over classes,
\begin{align*}
    \|\mu_i-s_i\|_2&=\bigl\|\E_{P_i}[z]-s_i\bigr\|_2
      \le\bigl(\E_{P_i}\|z-s_i\|_2^2\bigr)^{1/2},
    \\
    \frac1n\sum_i\bigl(\E_{P_i}\|z-s_i\|_2^2\bigr)^{1/2}
      &\le\Bigl(\frac1n\sum_i\E_{P_i}\|z-s_i\|_2^2\Bigr)^{1/2},
\end{align*}
and under balanced classes the bracketed average equals
$\E\|N_1(X)-s_Y\|_2^2=\tfrac{2}{\rho}\Lcon$. Combining these inequalities yields
$$
    \frac{1}{n} \sum_{i=1}^n \|\mu_i - s_i\|_2 \leq \sqrt{\frac{2}{\rho} \Lcon}
$$

\emph{Step 2: Non-Gaussianity term.} We compare both $P_i$ and $\Gamma_i$ with the point mass $\delta_{\mu_i}$ at their common mean. The only coupling with a point mass is the product coupling, so the transport cost is simply the mean distance to $\mu_i$:
\[
    \Delta_i\le W_1(P_i,\delta_{\mu_i})+W_1(\delta_{\mu_i},\Gamma_i)
      =\E_{P_i}\|z-\mu_i\|_2+\E_{\Gamma_i}\|z-\mu_i\|_2
      \le2\,(\operatorname{tr}\Sigma_i)^{1/2},
\]
by Jensen's inequality, since $P_i$ and $\Gamma_i$ have the same covariance $\Sigma_i$.  Using the bias-variance decomposition,
$\E_{P_i}\|z-s_i\|_2^2=\|\mu_i-s_i\|_2^2+\operatorname{tr}\Sigma_i$, the
Cauchy--Schwarz inequality in $\R^2$ (in the form $a + 2b \leq \sqrt{5}\sqrt{a^2 + b^2}$) gives
\[
    \|\mu_i-s_i\|_2+\Delta_i
      \le\|\mu_i-s_i\|_2+2(\operatorname{tr}\Sigma_i)^{1/2}
      \le\sqrt5\,\bigl(\E_{P_i}\|z-s_i\|_2^2\bigr)^{1/2},
\]
and also $\Delta_i\le2(\E_{P_i}\|z-s_i\|_2^2)^{1/2}$.  Averaging over $i$
exactly as for the mean term yields $\bar\Delta\le2\sqrt{2\Lcon/\rho}$ and the
$\sqrt5$ term of \eqref{eq:controlled-bound-full}.

\emph{Step 3: Covariance term.}  We separate the bias caused by the floor from the error of the covariance estimate. Since $\BW$ is a metric on positive semi-definite
matrices~\citep{bhatia2019bures},
\[
    \BW\bigl(\Sigma_i,\tilde\Sigma_i\bigr)
      \le \BW\bigl(\Sigma_i,\Sigma_i+\sigma_0^2I_k\bigr)
        + \BW\bigl(\Sigma_i+\sigma_0^2I_k,\hat\Sigma_i+\sigma_0^2I_k\bigr).
\]
For the first summand, the variational characterization
$\BW(A,B)=\min_{U\,\mathrm{unitary}}\|A^{1/2}-B^{1/2}U\|_F$
\citep{bhatia2019bures} with $U=I$ gives
$\BW(A,B)\le\|A^{1/2}-B^{1/2}\|_F$; because $\Sigma_i$ and
$\Sigma_i+\sigma_0^2I_k$ commute, this Frobenius norm equals
$\bigl(\sum_{j=1}^k(\sqrt{\lambda_j+\sigma_0^2}-\sqrt{\lambda_j})^2\bigr)^{1/2}
\le\sqrt{k}\,\sigma_0$,
where $\lambda_j\ge0$ are the eigenvalues of $\Sigma_i$ and we used
$\sqrt{\lambda+\sigma_0^2}-\sqrt{\lambda}\le\sigma_0$.

For the second summand, let $A:=\Sigma_i+\sigma_0^2I_k$ and
$B:=\hat\Sigma_i+\sigma_0^2I_k$; both dominate $\sigma_0^2I_k$, since
$\hat\Sigma_i$ is a sample covariance and hence positive semi-definite.
The idea is that the matrix square root is Lipschitz on matrices whose eigenvalues
are bounded from below by $\sigma_0^2$, with constant $1/(2\sigma_0)$.
Diagonalize $A=U\diag(\alpha)U^\top$ and $B=V\diag(\beta)V^\top$ with $U,V$
orthogonal and $\alpha_j,\beta_l\ge\sigma_0^2$, and set $C:=U^\top V$.  Then
$U^\top(A-B)V$ and $U^\top(A^{1/2}-B^{1/2})V$ have entries
$(\alpha_j-\beta_l)C_{jl}$ and $(\sqrt{\alpha_j}-\sqrt{\beta_l})C_{jl}$,
respectively.  Since
$\sqrt{\alpha}-\sqrt{\beta}=(\alpha-\beta)/(\sqrt{\alpha}+\sqrt{\beta})$ with
$\sqrt{\alpha}+\sqrt{\beta}\ge2\sigma_0$, and the Frobenius norm is invariant
under orthogonal transformations,
\[
    \|A^{1/2}-B^{1/2}\|_F^2
      =\sum_{j,l}\bigl(\sqrt{\alpha_j}-\sqrt{\beta_l}\bigr)^2C_{jl}^2
      \le\frac{1}{4\sigma_0^2}\sum_{j,l}(\alpha_j-\beta_l)^2C_{jl}^2
      =\frac{1}{4\sigma_0^2}\|A-B\|_F^2 .
\]
With $A-B=\Sigma_i-\hat\Sigma_i$ and $\BW\le\|A^{1/2}-B^{1/2}\|_F$ as above,
the second summand is at most $\tfrac{1}{2\sigma_0}\|\hat\Sigma_i-\Sigma_i\|_F$.

Inserting the bounds of Steps~1-3 into \eqref{eq:step0_inequality} yields
\eqref{eq:controlled-bound}; replacing the first two by their joint $\sqrt5$
bound yields \eqref{eq:controlled-bound-full}.
\end{proof}

We now discuss the terms on the right of \eqref{eq:controlled-bound} individually.
Each of them is either part of the lifted objective or explicitly managed by the
algorithm.
\begin{itemize}
  \item The \emph{sampled loss} $\Lsmp$ is the classification term of
        \eqref{eq:combinedLoss}, minimized directly;
  \item The \emph{mean mismatch} $\sqrt{2\Lcon/\rho}$ is the root-mean-square
        distance of the embeddings from their prototypes. The factor $\rho$ cancels,
        so this term does not depend directly on it. However, it is exactly what the
        consensus term penalizes. The annealing schedule \eqref{eq:anneal} increases $\rho$,
        and with it the weight placed on this distance, as training proceeds.
  \item The \emph{non-Gaussianity} $\bar\Delta$ is the modeling residual of
        assumption \eqref{eq:classAssumption}.  It is not optimized directly,
        but it is bounded by the same root-mean-square distance,
        $\bar\Delta\le2\sqrt{2\Lcon/\rho}$. Concentrating the embeddings around their prototypes
        therefore shrinks it whatever their shape, which is what \eqref{eq:controlled-bound-full} records.
        The bound \eqref{eq:controlled-bound} is the sharper of the two when the
        class-conditional laws are close to Gaussian, as the success of
        Gaussian models of penultimate features
        suggests~\citep{lee2018simple}. In addition, within-class variability is known to collapse late in
        training ~\citep{papyan2020prevalence}, which drives $\bar\Delta$ to zero;
  \item The \emph{covariance estimation error}
        $\|\hat\Sigma_i-\Sigma_i\|_F$ is the error of the per-epoch covariance
        refresh of Algorithm~\ref{alg:lifted}.  Covariance concentration bounds such as
        \eqref{eq:vershynin} control it when the refreshed embeddings are
        treated as independent draws from $P_i$ (a bound stated in operator
        norm converts via $\|\cdot\|_F\le\sqrt{k}\,\|\cdot\|_2$).  Because
        $N_1$ is fitted to the same data, this is an idealization rather than
        a guarantee, but it ties the sample-size discussion of
        \cref{sec:dimension} to a term of the risk bound;
  \item the \emph{floor bias} $\sqrt{k}\,\sigma_0$ is the price paid for the
        variance floor: the sampling law is wider than the embedding law.
        A larger floor increases this term but shrinks the estimation term, so the two trade off.
        With
        $\bar\varepsilon:=\frac1n\sum_i\|\hat\Sigma_i-\Sigma_i\|_F$, the sum
        $\sqrt{k}\,\sigma_0+\bar\varepsilon/(2\sigma_0)$ is minimized at
        $\sigma_0^\star=\bigl(\bar\varepsilon/(2\sqrt{k})\bigr)^{1/2}$, where it
        equals $\sqrt2\,k^{1/4}\bar\varepsilon^{1/2}$.  The bound itself thus
        favors an intermediate floor.
\end{itemize}
Two caveats apply.  First, the bound controls the cross-entropy risk and speaks to
accuracy only indirectly. A misclassified sample has predicted probability at most
$1/2$ for its true class and hence loss at least $\log2$, so the misclassification
probability is at most $R/\log2$.
Second, $\Lcon$ and $\Sigma_i$ are population
quantities. The bound holds deterministically for any parameters, but it says nothing
about generalization by itself. That question is moved into the gap between the
population and training values of the consensus term. 
With these caveats, discarding the
prototypes at inference is sound because the objective that trained $N_2$
was, up to these audited terms, the deployed risk itself.

\subsection{Sampling is an Explicit Curvature Regularizer}
\label{sec:sensitivity}

Corollary~\ref{cor:controlled} explains why reconnection does not degrade the
deployed risk. However, it does not explain why lifted training can \emph{gain}
accuracy over end-to-end training.  The gain has a separate source: the
sampled loss evaluates the head on entire neighborhoods of the prototypes
rather than at isolated points.  For our head this can only increase the loss
relative to the prototypes themselves, and to second order the increase is an
explicit, non-negative curvature penalty. The following proposition first states this for
a general smooth loss and then specializes to the linear head with cross-entropy.

\begin{proposition}[Second-order form of the sampled loss]
\label{prop:smoothing}
Let $\ell:\R^k\to\R$ be three times continuously differentiable with
$\bigl|\nabla^3\ell(z)[u,u,u]\bigr|\le M\|u\|_2^3$ for all $z,u\in\R^k$.
Then for $Q=\mathcal{N}(s,\Sigma)$,
\begin{equation}
    \Bigl|\,\E_{z\sim Q}\,\ell(z)-\ell(s)
      -\tfrac12\operatorname{tr}\bigl(\Sigma\,\nabla^2\ell(s)\bigr)\Bigr|
    \;\le\; \tfrac{M}{6}\,3^{3/4}\,(\operatorname{tr}\Sigma)^{3/2}.
    \label{eq:smoothing-expansion}
\end{equation}
For the linear head with cross-entropy (Remark~\ref{rem:lipschitz-ours}), each
$\ell_i$ is convex, its Hessian is $\nabla^2\ell_i(z)=W_2^\top H(z)W_2$ with
$H(z)=\diag\bigl(p(z)\bigr)-p(z)p(z)^\top\succeq0$ for every $z$, and the
hypothesis above holds with $M=\|W_2\|_2^3/\sqrt2$.  Consequently
\begin{equation}
    \E_{z\sim Q_i}\,\ell_i(z)\;\ge\;\ell_i(s_i)
    \label{eq:jensen-prototype}
\end{equation}
exactly, and
\begin{equation}
    \E_{z\sim Q_i}\,\ell_i(z)
      \;=\;\ell_i(s_i)
      \;+\;\underbrace{\tfrac12\operatorname{tr}\bigl(\tilde\Sigma_i\,
            W_2^\top H(s_i)W_2\bigr)}_{\ge\,0}
      \;+\;r_i,
    \qquad
    |r_i|\le\tfrac{3^{3/4}}{6\sqrt2}\,\|W_2\|_2^3\,
      (\operatorname{tr}\tilde\Sigma_i)^{3/2}.
    \label{eq:sensitivity-penalty}
\end{equation}
\end{proposition}

\begin{proof}
\emph{Step 1: the general expansion \eqref{eq:smoothing-expansion}.} Write $z=s+\zeta$ with $\zeta\sim\mathcal{N}(0,\Sigma)$.  Third-order Taylor
expansion of $t\mapsto\ell(s+t\zeta)$ with Lagrange remainder gives, for every
$\zeta$,
\[
    \Bigl|\ell(s+\zeta)-\ell(s)-\nabla\ell(s)^\top\zeta
      -\tfrac12\zeta^\top\nabla^2\ell(s)\zeta\Bigr|
    \le\tfrac{M}{6}\|\zeta\|_2^3 .
\]
We take expectations. The linear term vanishes, $\E[\nabla\ell(s)^\top\zeta]=0$,
and the quadratic term gives $\E[\tfrac12\zeta^\top\nabla^2\ell(s)\zeta]
=\tfrac12\operatorname{tr}(\Sigma\nabla^2\ell(s))$.  For the remainder,
H\"older's inequality and the Gaussian fourth moment
$\E\|\zeta\|_2^4=(\operatorname{tr}\Sigma)^2+2\operatorname{tr}(\Sigma^2)
\le3(\operatorname{tr}\Sigma)^2$ give
$\E\|\zeta\|_2^3\le(\E\|\zeta\|_2^4)^{3/4}
\le3^{3/4}(\operatorname{tr}\Sigma)^{3/2}$, proving
\eqref{eq:smoothing-expansion}.

\emph{Step 2: convexity and Hessian of the cross-entropy head.} Here $\ell_i(z)=-\log p_i(z)
=\operatorname{LSE}(W_2z+b)-(W_2z+b)_i$ with
$p(z)=\operatorname{softmax}(W_2z+b)$ and $\operatorname{LSE}$ the
log-sum-exp function.  As a convex function of an affine map minus an affine
map, $\ell_i$ is convex, and since $\E_{Q_i}[z]=s_i$, Jensen's inequality
gives \eqref{eq:jensen-prototype}.  The chain rule gives
$\nabla\ell_i(z)=W_2^\top(p(z)-e_i)$ and
$\nabla^2\ell_i(z)=W_2^\top\bigl(\diag(p(z))-p(z)p(z)^\top\bigr)W_2$.
The middle factor $H(z)$ is positive semi-definite: for any $v\in\R^n$,
$v^\top H(z)v=\sum_jp_jv_j^2-(\sum_jp_jv_j)^2\ge0$ by Jensen's inequality
(it is the variance of $v_J$ for $J\sim p$),
hence $\operatorname{tr}(\tilde\Sigma_i W_2^\top H(s_i)W_2)\ge0$ since the
trace of a product of two positive semi-definite matrices is non-negative.

\emph{Step 3: the constant $M$.}
Fix $z,u$ and let $w:=W_2u$.  Up to an affine
function of $t$, $t\mapsto\ell_i(z+tu)$ equals
$\log\sum_jp_j(z)e^{tw_j}$, the cumulant generating function of $V:=w_J$
with $J\sim p(z)$. Its third derivative at $t=0$ is the third cumulant, i.e.,
the third central moment, so  $\nabla^3\ell_i(z)[u,u,u]=\E[(V-\E V)^3]$.  With
$r:=\max_jw_j-\min_jw_j$ we have $|V-\E V|\le r$ and, by Popoviciu's
inequality, $\operatorname{Var}V\le r^2/4$, so
$|\E[(V-\E V)^3]|\le r^3/4$.  Finally
$r\le\sqrt2\,\|w\|_2\le\sqrt2\,\|W_2\|_2\|u\|_2$, which gives
$M=\|W_2\|_2^3/\sqrt2$.  Equation \eqref{eq:sensitivity-penalty} follows from
\eqref{eq:smoothing-expansion} applied to $\ell_i$ at $s=s_i$,
$\Sigma=\tilde\Sigma_i$.
\end{proof}

\paragraph{Interpretation}
Equation~\eqref{eq:sensitivity-penalty} is the class-prototype analogue of the
classical result that training with input noise is, to second order,
equivalent to Tikhonov regularization~\citep{bishop1995training}: the sampled
loss equals the prototype loss plus a penalty on the $\tilde\Sigma_i$-weighted
curvature of the head at the prototype.
Fitting the prototypes constrains alone would constrain only the values
$\ell_i(s_i)$. Among heads that fit them equally well, the sampled
objective prefers the one whose loss curves least around them, in the
directions of $\tilde\Sigma_i$.  These are exactly the directions that
matter at inference, where the head receives $N_1(x)=s_i+\delta$ with a
nonzero reconnection residual $\delta$ instead of the prototype itself.

\paragraph{Two qualifications.} This reading has two limits.  First, it is a second-order
statement.  The remainder constant in \eqref{eq:sensitivity-penalty} grows
like $\|W_2\|_2^3$, so ranking heads by the curvature term alone is justified
only when the logit perturbation scale
$\|W_2\|_2(\operatorname{tr}\tilde\Sigma_i)^{1/2}$ is small.  Second, curvature
at the prototype is not a global sensitivity measure.  Writing
$\omega_j^\top$ for the rows of $W_2$ and
$\bar\omega:=\sum_jp_j(s_i)\,\omega_j$,
\begin{equation}
    \operatorname{tr}\bigl(W_2^\top H(s_i)W_2\bigr)
      =\sum_{j=1}^n p_j(s_i)\,\|\omega_j-\bar\omega\|_2^2,
    \label{eq:curvature-spread}
\end{equation}
the $p(s_i)$-weighted spread of the class weight vectors.  A head that is
saturated at the prototype, $p(s_i)\approx e_i$, makes this spread small even
when $\|W_2\|_2$, and with it the Lipschitz constant $G$, is large.  Scaling
up $W_2$ in fact drives the curvature penalty to zero exponentially fast.
That regime lies outside the validity of the expansion.  What keeps the head
out of it is the weight decay on $N_2$, which also controls $G$ in
\eqref{eq:controlled-bound}.

\paragraph{Relation to end-to-end training.}  Proposition~\ref{prop:smoothing} compares the
sampled loss with the loss at the prototypes, not with the end-to-end loss
$\E_{P_i}\ell_i$ on the embeddings themselves.  
A direct comparison with the latter is possible in the idealized case 
$s_i=\mu_i$, $\hat\Sigma_i=\Sigma_i$.  Then
$Q_i$ is the convolution of $\Gamma_i=\mathcal{N}(\mu_i,\Sigma_i)$ with
$\mathcal{N}(0,\sigma_0^2I_k)$, and convexity of $\ell_i$ gives
\[
    \E_{Q_i}\ell_i
      =\E_{z\sim\Gamma_i}\,\E_{\xi\sim\mathcal{N}(0,\sigma_0^2I_k)}\,
        \ell_i(z+\xi)
      \;\ge\;\E_{\Gamma_i}\ell_i
      \;\ge\;\E_{P_i}\ell_i-G\,\Delta_i ,
\]
where the last step is the coupling argument of Theorem~\ref{thm:risk-transfer}.
Applying \eqref{eq:smoothing-expansion} to both Gaussians shows that, to
second order, the excess of the sampled loss over $\E_{\Gamma_i}\ell_i$ is
$\tfrac{\sigma_0^2}{2}\operatorname{tr}(W_2^\top H(s_i)W_2)$.  Relative to
end-to-end training, the additional regularization therefore comes from the
floor part of $\tilde\Sigma_i$, not from the whole of $\tilde\Sigma_i$.
Outside the idealized case, the mean and covariance mismatches of
Corollary~\ref{cor:controlled} enter as well.

\paragraph{The floor keeps smoothing alive under collapse}
Since $Q_i$ is also the convolution of $\mathcal{N}(s_i,\sigma_0^2I_k)$ with
$\mathcal{N}(0,\hat\Sigma_i)$, the same convexity argument gives
\begin{equation}
    \E_{Q_i}\ell_i\;\ge\;\E_{\xi\sim\mathcal{N}(0,\sigma_0^2I_k)}\,
      \ell_i(s_i+\xi),
    \label{eq:floor-smoothing}
\end{equation}
and, to second order, since $\tilde\Sigma_i\succeq\sigma_0^2I_k$ and
$W_2^\top H(s_i)W_2\succeq0$,
\begin{equation}
    \tfrac12\operatorname{tr}\bigl(\tilde\Sigma_i W_2^\top H(s_i)W_2\bigr)
      \;\ge\;\tfrac{\sigma_0^2}{2}\,
      \operatorname{tr}\bigl(W_2^\top H(s_i)W_2\bigr).
    \label{eq:floor-regularization}
\end{equation}
Even in the neural-collapse limit $\hat\Sigma_i\to0$, in which the estimated
covariance degenerates, the sampled loss therefore remains an isotropically
smoothed loss at scale $\sigma_0$ rather than reducing to the point loss
$\ell_i(s_i)$.  This resolves the tension noted in \cref{sec:intro}:
within-class collapse makes the covariance \emph{estimate} degenerate, but the
floor converts the degenerate limit into fixed isotropic smoothing rather than
none.  The floor alone, however, does not keep the penalty bounded away from zero.
By \eqref{eq:curvature-spread}, the right side of
\eqref{eq:floor-regularization} vanishes as the head saturates.  The two
safeguards act together: the floor fixes the smoothing scale, and weight decay
on $N_2$ bounds $\|W_2\|_2$ and hence the degree of saturation.

%

\subsection{An Exactly Solvable One-Dimensional Instance}
\label{sec:one-dimensional-mechanism}

The results of \cref{sec:risk-transfer,sec:sensitivity} are stated as bounds
with remainder terms.  This subsection exhibits the selection effect of
\cref{prop:smoothing} exactly, in the smallest setting in which it occurs:
one embedding dimension, two classes, and a one-parameter family of heads.

The example holds the first subnetwork fixed and looks only at the
classification problem that the second subnetwork has to solve. 
It makes no claim about training $N_1$ or about optimization. What it does
show, in closed form, are the following things:
\begin{itemize}
    \item fitting the class prototypes alone leaves one degree of freedom of the head completely
    undetermined;
    \item sampling in a neighborhood of the prototypes removes this ambiguity;
    \item for small noise, the head it picks out is the one whose output is least sensitive to a perturbation of its input.
\end{itemize}

\subsubsection{Setup}

Consider balanced binary classification with labels $Y\in\{-1,+1\}$.  Take the
lifting dimension to be $k=1$ and assume that the output of the fixed first
subnetwork is
\begin{equation}
    N_1(X)=Y+\varepsilon,
    \qquad \varepsilon\sim\mathcal{N}(0,\sigma^2),
    \qquad \varepsilon\ \text{independent of}\ Y ,
    \label{eq:toy-embedding}
\end{equation}
so that the two class prototypes are $s_{-}=-1$ and $s_{+}=+1$.  This is the
one-dimensional, homoscedastic instance of the class-conditional model
 \eqref{eq:classAssumption}; the variance floor of \eqref{eq:shrinkage} can be
absorbed into $\sigma^2$.  The head has a scalar output, and we use the squared loss
$(N_2(z) - y)^2$ for a sample with label $y\in \{-1, 2\}$. In place of the linear head we consider
the one-parameter family
\begin{equation}
    N_2(z;a)=g_a(z):=z+a\,(z^3-z),
    \qquad a\in\R .
    \label{eq:toy-head}
\end{equation}

This family is chosen because the factor $z^3-z$ vanishes at both
prototypes.  Every member therefore reproduces the prototypes exactly,
\begin{equation}
    g_a(1)=1,\qquad g_a(-1)=-1
    \qquad\text{for all }a\in\R ,
    \label{eq:toy-interpolation}
\end{equation}
while the members differ in how they behave \emph{near} the prototypes, since
\begin{equation}
    g_a'(z)=1+a\,(3z^2-1),
    \qquad\text{so}\qquad
    g_a'(\pm1)=1+2a .
    \label{eq:toy-derivative}
\end{equation}
The parameter $a$ is thus invisible to any criterion that only inspects the prototypes, and it directly sets the slope $1+2a$ of the head there. We call this slope the \emph{local sensitivity} of the head: it is the factor by which a small perturbation of the input is passed on to the output.

To separate ``fitting the prototypes'' from ``fitting their neighborhoods'', define
\begin{align}
    L_{\mathrm{point}}(a)
      &:=\E_Y\bigl[(g_a(Y)-Y)^2\bigr],
      \label{eq:toy-point-loss}\\
    L_{\mathrm{sample}}(a)
      &:=\E_{Y,\varepsilon}\bigl[(g_a(Y+\varepsilon)-Y)^2\bigr].
      \label{eq:toy-sampled-loss}
\end{align}
The first is the zero-variance limit $\sigma\downarrow0$, in which $N_2$ sees
only the two prototypes.
The second is the population version of the classification term in
\eqref{eq:combinedLoss}--\eqref{eq:reparam}.

\subsubsection{The Selection Result}

\begin{proposition}[Sampling determines the local sensitivity]
\label{prop:toy-noise-selection}
Assume \eqref{eq:toy-embedding}--\eqref{eq:toy-head}.  Then
\begin{enumerate}
  \item[(i)] $L_{\mathrm{point}}(a)=0$ for every $a\in\R$; the prototype loss
        does not identify $a$.
  \item[(ii)] For every $\sigma>0$ the sampled loss is the strictly convex
        quadratic
        \begin{equation}
            L_{\mathrm{sample}}(a)
              = \sigma^2
                +2a\,(2\sigma^2+3\sigma^4)
                +a^2(4\sigma^2+39\sigma^4+15\sigma^6),
            \label{eq:toy-sampled-expanded}
        \end{equation}
        with the unique minimizer
        \begin{equation}
            a^\star(\sigma)
              =-\frac{2+3\sigma^2}{4+39\sigma^2+15\sigma^4}.
            \label{eq:toy-optimal-a}
        \end{equation}
  \item[(iii)] The sensitivity of the selected head at the prototypes is
        \begin{equation}
            g_{a^\star}'(\pm1)
              =\frac{33\sigma^2+15\sigma^4}{4+39\sigma^2+15\sigma^4}
              \ \in(0,1),
            \qquad
            g_{a^\star}'(\pm1)=\tfrac{33}{4}\sigma^2+O(\sigma^4)
            \ \longrightarrow\ 0
            \quad\text{as }\sigma\downarrow0 .
            \label{eq:toy-flatness}
        \end{equation}
\end{enumerate}
In words: fitting the prototypes leaves the sensitivity of the head
undetermined, whereas sampling any neighborhood of them fixes it, and in the
small-noise limit the value selected is the one that makes the head stationary
at both prototypes.
\end{proposition}

\begin{proof}
Part (i) is immediate from \eqref{eq:toy-interpolation}.

For (ii), put $q:=\sigma^2$ and $Z:=Y+\varepsilon$.  Expanding the cube and
using $Y^2=1$, hence $Y^3=Y$,
\begin{align}
    Z^3-Z
      &=(Y+\varepsilon)^3-(Y+\varepsilon) \nonumber\\
      &=Y^3+3Y^2\varepsilon+3Y\varepsilon^2+\varepsilon^3-Y-\varepsilon
        \nonumber\\
      &=2\varepsilon+3Y\varepsilon^2+\varepsilon^3 ,
    \label{eq:toy-cubic-residual}
\end{align}
so that the residual of the head is
\begin{equation}
    g_a(Z)-Y
      =\varepsilon+a\,U,
    \qquad
    U:=2\varepsilon+3Y\varepsilon^2+\varepsilon^3 .
    \label{eq:toy-total-residual}
\end{equation}
Because $\varepsilon$ is centred Gaussian and independent of $Y$, and because
$\E[Y]=0$ by balance, every monomial containing an odd power of $\varepsilon$
or an odd power of $Y$ has zero expectation.  With the Gaussian moments
\begin{equation}
    \E[\varepsilon^2]=q,\qquad
    \E[\varepsilon^4]=3q^2,\qquad
    \E[\varepsilon^6]=15q^3,
\end{equation}
the two remaining expectations are
\begin{align}
    \E[\varepsilon U]
      &=\E\bigl[2\varepsilon^2+3Y\varepsilon^3+\varepsilon^4\bigr]
       =2q+3q^2,\\
    \E[U^2]
      &=\E\bigl[4\varepsilon^2+12Y\varepsilon^3
                +(9Y^2+4)\varepsilon^4+6Y\varepsilon^5+\varepsilon^6\bigr]
       =4q+39q^2+15q^3 .
\end{align}
Squaring \eqref{eq:toy-total-residual} and taking expectations gives
$L_{\mathrm{sample}}(a)=q+2a(2q+3q^2)+a^2(4q+39q^2+15q^3)$, which is
\eqref{eq:toy-sampled-expanded}.  The coefficient of $a^2$ is
$\E[U^2]>0$ for every $q>0$, so the quadratic is strictly convex and its
stationary point is its unique minimizer:
\begin{equation}
    a^\star
      =-\frac{\E[\varepsilon U]}{\E[U^2]}
      =-\frac{2q+3q^2}{4q+39q^2+15q^3}
      =-\frac{2+3q}{4+39q+15q^2},
\end{equation}
which is \eqref{eq:toy-optimal-a}.  Part (iii) follows from
\eqref{eq:toy-derivative}:
\begin{equation}
    g_{a^\star}'(\pm1)=1+2a^\star
      =\frac{(4+39q+15q^2)-(4+6q)}{4+39q+15q^2}
      =\frac{33q+15q^2}{4+39q+15q^2},
\end{equation}
which lies in $(0,1)$ for every $q>0$ and equals $\tfrac{33}{4}q+O(q^2)$ as
$q\downarrow0$.  Substituting $q=\sigma^2$ completes the proof.
\end{proof}

\subsubsection{Interpretation}

The proposition separates two things that the full algorithm does at once.
Identity \eqref{eq:toy-interpolation} says that classifying the prototypes
perfectly puts no constraint at all on how $N_2$ responds to the \emph{actual}
output of $N_1$, which is a noisy version of a prototype.  The expectation in
\eqref{eq:toy-sampled-loss} supplies exactly the missing constraint.

As $\sigma\downarrow0$ the minimizer tends to $a^\star=-1/2$, that is, to
\begin{equation}
    g_{-1/2}(z)=\tfrac32z-\tfrac12z^3,
    \qquad g_{-1/2}'(\pm1)=0 .
\end{equation}
This is \cref{prop:smoothing} made exact: expanding
\eqref{eq:toy-sampled-expanded} to leading order gives
$L_{\mathrm{sample}}(a)=\sigma^2(1+2a)^2+O(\sigma^4)
=\sigma^2\,g_a'(\pm1)^2+O(\sigma^4)$, which is precisely the sensitivity
penalty $\tfrac12\operatorname{tr}(\Sigma\nabla^2\ell(s))$ of
\eqref{eq:smoothing-expansion} in this model (the prototype loss vanishes by
\eqref{eq:toy-interpolation}), and its minimizer is the flattest head.  The
higher-order terms in \eqref{eq:toy-sampled-expanded} are the remainder of
\eqref{eq:smoothing-expansion}, here available in closed form.
The reason this matters is visible in the response to a small deterministic
reconnection error $\delta$.  Setting $\varepsilon=\delta$ in
\eqref{eq:toy-total-residual},
\begin{equation}
    g_a(Y+\delta)-Y
      =(1+2a)\,\delta+3aY\,\delta^2+a\,\delta^3 .
    \label{eq:toy-local-error}
\end{equation}
The identity head $a=0$ makes an error of order $\delta$; the limiting sampled
solution $a=-1/2$ cancels the linear term and makes an error of order
$\delta^2$, hence a squared loss of order $\delta^4$ rather than $\delta^2$.
This is an exact toy instance of feature-space noise acting as a sensitivity
regularizer for the second subnetwork.

\begin{remark}[The selected head is matched to the noise level]
\label{rem:toy-matched}
Sampling does not always favour a flat head.  By \eqref{eq:toy-optimal-a},
$a^\star(\sigma)\to0$ as $\sigma\to\infty$, so at large noise the sampled
objective returns to the identity head, and $g_{a^\star}'(\pm1)$ increases
monotonically from $0$ to $1$ as $\sigma$ grows.  The mechanism selects the
head matched to the noise level, and flatness is the small-noise end of that
family rather than the target of the objective.
\end{remark}

\begin{remark}[Scope of the conclusion]
\label{rem:toy-monotonicity}
The conclusion is deliberately local, for two reasons.

First, it is not a statement about clean accuracy.  The head $g_a$ is odd, so
whenever it is increasing on the interval that carries the data, its decision
boundary is $z=0$ and its accuracy equals that of the identity head.  By
\eqref{eq:toy-derivative}, however, $g_a$ with $a<0$ is increasing only on
$|z|\le\sqrt{(1-a)/(-3a)}$ and changes sign again at $|z|=\sqrt{1-1/a}$.  For
the limiting head $a=-1/2$ these radii are $1$ and $\sqrt3$: a sample with
$|Z|>\sqrt3$ is misclassified even though it is on the correct side of the
origin.  Since $Z$ is Gaussian, this happens with positive probability, and at
moderate noise it is not negligible --- \cref{sec:toy-numerics} measures a loss
of roughly seven accuracy points at $\sigma=0.5$ for $a=-1/2$.  The
$\sigma$-matched head $a^\star(\sigma)$ is far milder, because $|a^\star|$
shrinks as $\sigma$ grows (\cref{rem:toy-matched}); its accuracy stays within
$0.15$ points of the identity for $\sigma\le0.5$ and within $0.6$ points over
the whole range tested.

Second, the cubic should not be extrapolated to large $|z|$, where its
behavior is not representative of a practical classifier.
\end{remark}

What \cref{prop:toy-noise-selection} does establish is a concrete mechanism by
which the \emph{changed} lifted objective --- rather than faster optimization
of the same objective --- selects a different and locally more robust solution.
The experiments below test whether this mechanism mediates the accuracy gains
of \cref{sec:results}.

\subsubsection{Illustration and Numerical Verification}
\label{sec:toy-numerics}

The proposition can be checked without training a network.  On a grid of noise
levels $\sigma$ we draw $n=2\cdot10^5$ samples $(Y,Z)$ from
\eqref{eq:toy-embedding} and minimize the empirical version of
\eqref{eq:toy-sampled-loss}.  Because that empirical loss is again quadratic in
$a$, its minimizer is available in closed form,
\begin{equation}
    \hat a
      = -\frac{\sum_i (Z_i-Y_i)\,(Z_i^3-Z_i)}{\sum_i (Z_i^3-Z_i)^2},
    \label{eq:toy-empirical-minimizer}
\end{equation}
which we cross-check against a derivative-free line search.  Each configuration
is repeated over $25$ seeds; the reported error bars are standard deviations
across seeds.  A held-out sample of the same size is used for the reported
losses and accuracies.

Across $\sigma\in[0.02,1]$ the estimate $\hat a$ agrees with
\eqref{eq:toy-optimal-a} and the measured sensitivity $|g_{\hat a}'(1)|$ agrees
with \eqref{eq:toy-flatness} to within one seed standard deviation; the
empirical loss agrees with the closed form \eqref{eq:toy-sampled-expanded} to
within Monte-Carlo error.  \Cref{fig:toy-selection} shows both comparisons, and
\cref{fig:toy-perturbation} shows the response to a bounded reconnection error.

\begin{figure}[t]
  \centering
  \includegraphics[width=\linewidth]{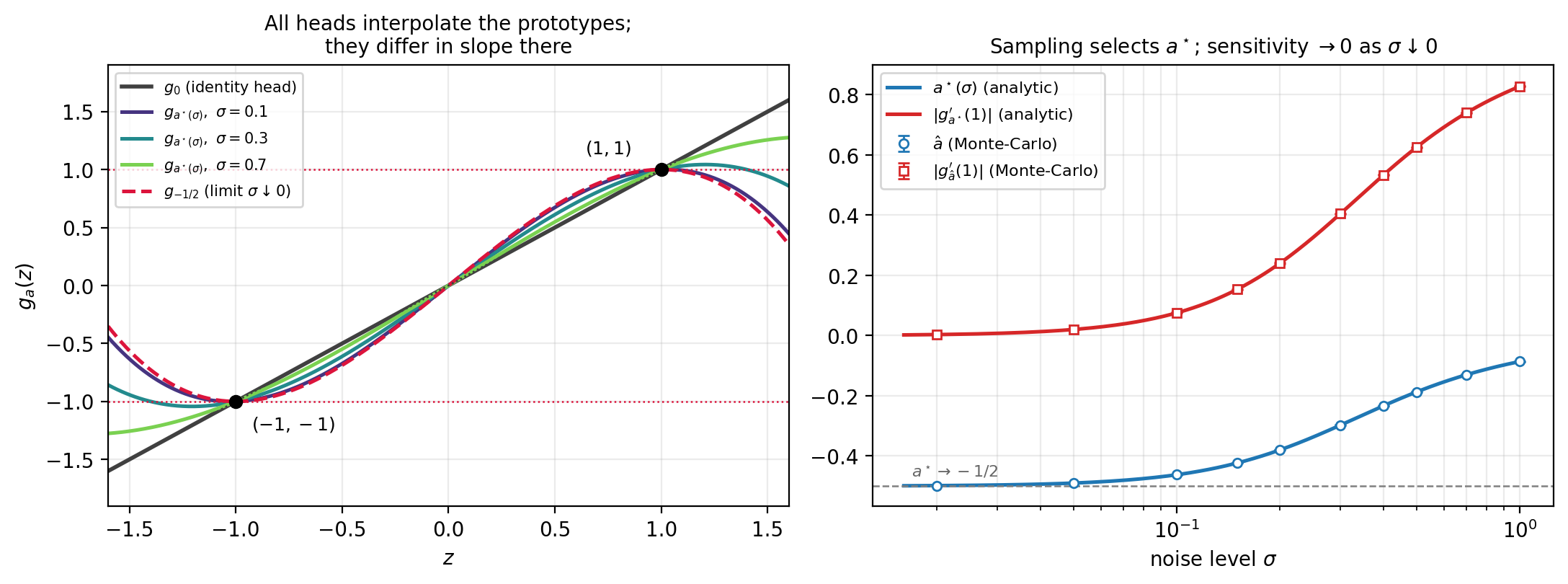}
  \caption{%
    \textbf{Left:} the identity head $g_0$, the sampled optima
    $g_{a^\star(\sigma)}$ for $\sigma\in\{0.1,0.3,0.7\}$, and the small-noise
    limit $g_{-1/2}$.  All curves pass through $(-1,-1)$ and $(1,1)$; the
    dotted lines are the tangents there, which are the only thing that
    distinguishes the family.
    \textbf{Right:} Monte-Carlo estimates of $a^\star(\sigma)$ and of
    $|g_{a^\star}'(1)|$ against the closed forms
    \eqref{eq:toy-optimal-a} and \eqref{eq:toy-flatness}.  Markers are means
    over $25$ seeds and error bars are one standard deviation; they are smaller
    than the markers at most noise levels.}
  \label{fig:toy-selection}
\end{figure}

\begin{figure}[t]
  \centering
  \includegraphics[width=\linewidth]{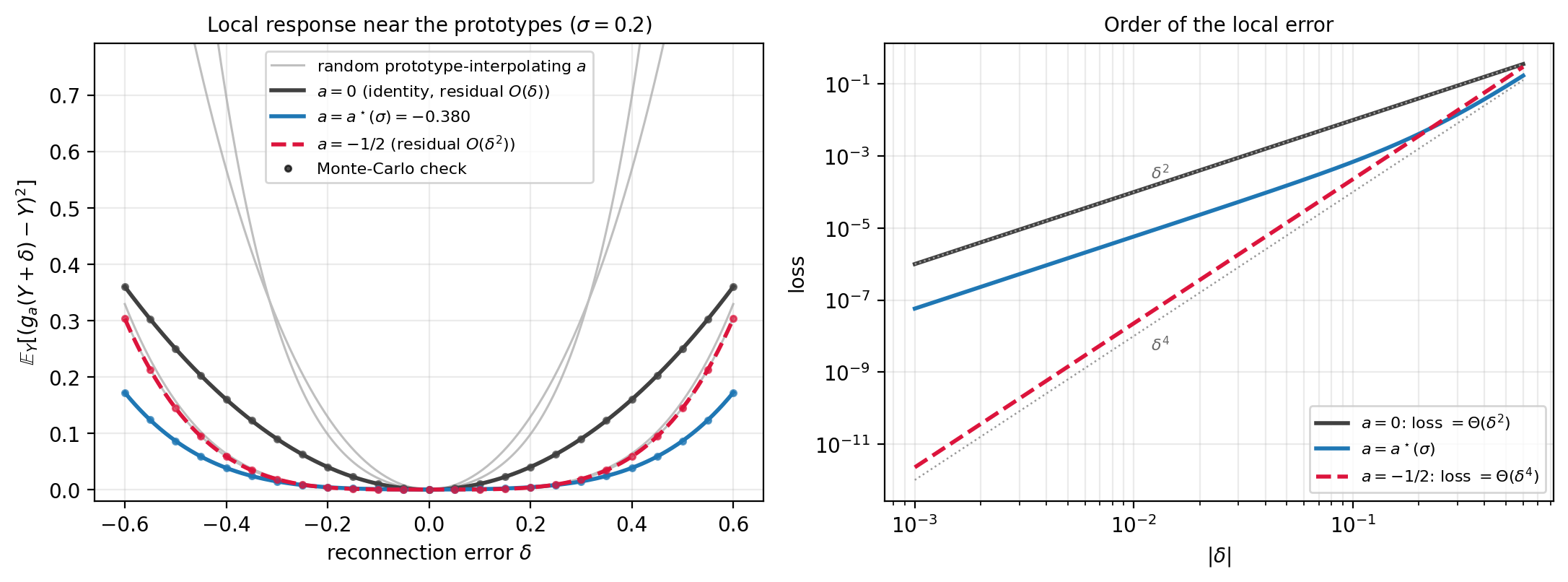}
  \caption{%
    \textbf{Left:} loss $\E_Y[(g_a(Y+\delta)-Y)^2]$ under a bounded deterministic
    reconnection error $\delta$, restricted to the neighborhood of the
    prototypes for which the local model is intended, at $\sigma=0.2$.  Grey
    curves are randomly drawn prototype-interpolating values of $a$, which
    illustrate that interpolation alone constrains nothing.
    \textbf{Right:} the same curves on log--log axes, showing the
    $\Theta(\delta^2)$ behavior of the identity head against the
    $\Theta(\delta^4)$ behavior of $g_{-1/2}$ predicted by
    \eqref{eq:toy-local-error}.}
  \label{fig:toy-perturbation}
\end{figure}

\end{document}